\documentclass[a4paper,10pt]{article}

\usepackage[top=30truemm,bottom=30truemm,left=25truemm,right=25truemm]{geometry}
\usepackage[dvipdfmx]{graphicx}

\usepackage{amsmath,amssymb,amsthm,mathtools}
\usepackage{bm}
\usepackage{enumitem}
\usepackage{booktabs}
\usepackage{natbib}
\usepackage{thmtools}
\usepackage{thm-restate}
\usepackage{hyperref}
\hypersetup{colorlinks=true,linkcolor=blue,citecolor=blue,urlcolor=blue}
\usepackage{mathtools}
\usepackage{bbm}
\usepackage{color}

\declaretheorem[name=Assumption]{assumption}
\declaretheorem[name=Definition]{definition}
\declaretheorem[name=Lemma]{lemma}
\declaretheorem[name=Proposition]{proposition}
\declaretheorem[name=Theorem]{theorem}

\declaretheorem[name=Remark]{remark}

\title{Asymptotic Theory and Phase Transitions for Variable Importance in Quantile Regression Forests}
\author{Tomoshige Nakamura \thanks{Juntendo University} ,  Hiroshi Shiraishi \thanks{Keio University}}
\date{\today}

\begin{document}
\maketitle

\begin{abstract}
Quantile Regression Forests (QRF) are widely used for non-parametric conditional quantile estimation, yet statistical inference for variable importance measures remains challenging due to the non-smoothness of the loss function and the complex bias-variance trade-off. In this paper, we develop a asymptotic theory for variable importance defined as the difference in pinball loss risks. We first establish the asymptotic normality of the QRF estimator by handling the non-differentiable pinball loss via Knight's identity. Second, we uncover a "phase transition" phenomenon governed by the subsampling rate $\beta$ (where $s \asymp n^{\beta}$). We prove that in the bias-dominated regime ($\beta \ge 1/2$), which corresponds to large subsample sizes typically favored in practice to maximize predictive accuracy, standard inference breaks down as the estimator converges to a deterministic bias constant rather than a zero-mean normal distribution. Finally, we derive the explicit analytic form of this asymptotic bias and discuss the theoretical feasibility of restoring valid inference via analytic bias correction. Our results highlight a fundamental trade-off between predictive performance and inferential validity, providing a theoretical foundation for understanding the intrinsic limitations of random forest inference in high-dimensional settings.
\end{abstract}

\section{Introduction}

Quantile regression plays a pivotal role in modern statistical analysis, offering a comprehensive view of the conditional distribution of a response variable beyond the conditional mean. In fields ranging from financial risk management to environmental science, understanding the tail behavior and heterogeneity of the outcome distribution is often more critical than predicting central trend. To estimate conditional quantiles in high-dimensional and complex settings, non-parametric machine learning methods have gained immense popularity. Among them, Quantile Regression Forests (QRF), introduced by \citep{Meinshausen2006} and theoretically refined by \citep{Athey2019}, stand out for their adaptability to local structures and robustness against the curse of dimensionality.

The practical utility of Quantile Regression Forests (QRF) has been demonstrated across a diverse range of disciplines, particularly where quantifying uncertainty and managing risk are as vital as point prediction. In precision medicine, QRF is employed to predict drug responses and assess prediction reliability, enabling more personalized therapeutic strategies \citep{fang2018quantile}. In climate science, it serves as a powerful tool for forecasting extreme weather events, such as heat waves, by modeling the conditional quantiles of temperature distributions \citep{khan2019prediction}. Furthermore, in the building energy sector, QRF is utilized to predict the uncertainty of chiller plant power consumption, providing confidence intervals that are crucial for efficient demand response and energy management \citep{rizi2023predicting}. Similarly, in quantitative finance, QRF models have been developed to forecast financial risk measures like Value-at-Risk (VaR) and Expected Shortfall (ES), effectively capturing non-linear dependencies and variable interactions in high-dimensional market data \citep{james2023forecasting}.

Despite their predictive success, a significant challenge remains: interpretability. The black-box nature of random forests obscures the contribution of individual covariates to the estimated quantiles. In decision-making contexts, it is crucial not only to predict risk but also to identify which factors drive it. While various heuristic measures of variable importance (VI) exist—such as permutation importance \citep{Breiman2001} or mean decrease impurity—they often lack statistical properties, making it difficult to construct valid confidence intervals or conduct hypothesis tests.

Recent advances in semi-parametric theory, notably by \citep{Williamson2023}, have established a general framework for algorithm-agnostic variable importance inference. Their approach typically relies on the efficiency of cross-fitting and the assumption that the nuisance estimators (the machine learning models) converge sufficiently fast—specifically, faster than $n^{-1/4}$. While this condition holds for parametric models or low-dimensional smooth functions, it is often violated by non-parametric estimators like random forests in high-dimensional settings. As the dimensionality increases, the convergence rate of the forest deteriorates, causing the bias inherent in the estimator to dominate the sampling fluctuation. Consequently, standard inference based on asymptotic normality breaks down, leading to invalid confidence intervals and inflated Type I errors.

In this paper, we develop a asymptotic theory for variable importance in quantile regression forests, explicitly addressing the challenges posed by non-smooth loss functions and slow convergence rates. Our contribution is threefold.

First, we establish the asymptotic normality of the QRF estimator and derive the explicit form of limitting distribution. Unlike standard regression forests, the QRF minimizes the non-differentiable pinball loss. To handle this non-smoothness, we employ Knight's identity \citep{Knight1998}, extending the analysis of honest forests by \citep{Wager2018, Athey2019} to the quantile setting. We derive the explicit convergence rate of the estimator as a function of the subsample size $s \asymp n^\beta$.

Second, we uncover a phase transition in the asymptotic behavior of the variable importance estimator governed by the subsampling rate $\beta$. We define variable importance as the difference in population risk (pinball loss) between a full model and a restricted model excluding the variable of interest. We show that:
\begin{itemize}
    \item In the variance-dominated regime ($\beta < 1/2$), the nuisance estimation error is negligible, and the VI estimator achieves standard $\sqrt{n}$-consistency and asymptotic normality.
    \item In the bias-dominated regime ($\beta \geq 1/2$), which corresponds to the large subsample sizes typically favored in practice to maximize accuracy, the bias of the QRF estimator decays too slowly. The VI estimator converges to a deterministic bias constant rather than zero, invalidating standard inference.
\end{itemize}

Third, based on our derivation of the explicit bias term, we discuss the theoretical implications for bias correction. We show that the deterministic bias characteristic of random forests in the regime $\beta > 1/2$ can be explicitly characterized using the forest's adaptive weights. This theoretical result suggests that while standard inference fails for large subsamples, asymptotic normality is recoverably in principle through analytic bias correction, provided that the nuisance components are estimated consistently.

The remainder of this paper is organized as follows. Section 2 defines the problem setting and the variable importance measure. Section 3 establishes the convergence rates and asymptotic normality of the QRF estimator. Section 4 presents the main theoretical results on variable importance, including the phase transition phenomenon and the derivation of the asymptotic bias. Section 5 concludes with a discussion on the practical implications of our theoretical findings.

%
%

\section{Preliminaries and Definitions}
\label{sec:preliminaries}

In this paper, we are interested in the conditional $\tau$-quantile of a response variable $Y \in \mathbb{R}$ given a set of covariates $X = (X_1, \dots, X_p)^\top \in \mathcal{X} \subset \mathbb{R}^p$. Let $Z = (X, Y)$ be a random vector, and we observe $n$ independent and identically distributed (i.i.d.) copies $Z_1, \dots, Z_n$.
For a quantile level $\tau \in (0, 1)$, we define the pinball loss function $\rho_{\tau}(u)$ and its associated score function $\psi_{\tau}(u)$ as:
\begin{align*}
    \rho_{\tau}(u) &:= u (\tau - \bm{1}\{u < 0\}), \\
    \psi_{\tau}(y, \theta) &:= \tau - \bm{1}\{y \le \theta\}.
\end{align*}

\subsection{Conditional Quantile Functions}

The conditional $\tau$-quantile function $q_{\tau}: \mathbb{R}^p \to \mathbb{R}$ is defined as the minimizer of the expected pinball loss:
\begin{equation*}
    q_{\tau}(x) \in \mathop{\mathrm{argmin}}_{\theta \in \mathbb{R}} \mathbb{E}[\rho_{\tau}(Y - \theta) \mid X = x].
\end{equation*}
For any subset of indices $S \subset \{1, \dots, p\}$, let $X_S = (X_j)_{j \in S}$ denote the sub-vector of covariates included in the model, and let $X_{-S}$ denote the complement vector corresponding to the indices $-S := \{1, \dots, p\} \setminus S$. We define the set of measurable functions depending only on $X_{-S}$ as $\mathcal{G}_{-S} := \{g: \mathbb{R}^p \to \mathbb{R} \mid g(x) = \tilde{g}(x_{-S}) \text{ for some } \tilde{g}: \mathbb{R}^{p-|S|} \to \mathbb{R}\}$.
The restricted conditional quantile function $q_{\tau, -S}(x)$, which does not utilize the information in $X_S$, is defined as:
\begin{equation*}
    q_{\tau, -S}(x) \in \mathop{\mathrm{argmin}}_{g \in \mathcal{G}_{-S}} \mathbb{E}[\rho_{\tau}(Y - g(X))].
\end{equation*}
Note that $q_{\tau, -S}(x)$ effectively depends only on $x_{-S}$.

\subsection{Variable Importance Measure for Quantiles}

We quantify the variable importance of the feature subset $X_S$ as the increase in population risk when $X_S$ is excluded from the model. The population risk functional $\mathcal{R}_{\tau}(f)$ for a prediction function $f$ is defined as:
\begin{equation*}
    \mathcal{R}_{\tau}(f) := \mathbb{E}[\rho_{\tau}(Y - f(X))].
\end{equation*}

\begin{definition}[Variable Importance]
The global variable importance of the subset $S$ at quantile level $\tau$, denoted by $V_{\tau}(S)$, is defined as the excess risk incurred by the restricted model relative to the full model:
\begin{equation*}
    V_{\tau}(S) := \mathcal{R}_{\tau}(q_{\tau, -S}) - \mathcal{R}_{\tau}(q_{\tau}) = \mathbb{E}[\rho_{\tau}(Y - q_{\tau, -S}(X)) - \rho_{\tau}(Y - q_{\tau}(X))].
\end{equation*}
\end{definition}

We also define the local variable importance at a point $x$ as:
\begin{equation*}
    v_{\tau}(S; x) := \mathbb{E}[\rho_{\tau}(Y - q_{\tau, -S}(X)) - \rho_{\tau}(Y - q_{\tau}(X)) \mid X = x].
\end{equation*}
By the law of iterated expectations, the global importance is the expectation of the local importance: $V_{\tau}(S) = \mathbb{E}[v_{\tau}(S; X)]$.
The variable importance measure defined above satisfies the following natural properties.

\begin{proposition}[Properties of Variable Importance]
\noindent
\begin{enumerate}
    \item[(i)] Non-negativity: $V_{\tau}(S) \ge 0$, since $q_{\tau}$ minimizes the global risk over a larger function space than $q_{\tau, -S}$.
    \item[(ii)] Normalization: If $S = \emptyset$, then $V_{\tau}(\emptyset) = 0$.
    \item[(iii)] Monotonicity: If $S \subset T$, then $V_{\tau}(S) \le V_{\tau}(T)$, as removing more variables cannot decrease the risk.
\end{enumerate}
\end{proposition}

\subsection{Neyman Orthogonality}

The variable importance $V_{\tau}(S)$ is a functional of the nuisance parameters $q_{\tau}$ and $q_{\tau, -S}$. Since these functions must be estimated from data, it is crucial to understand how their estimation errors affect the estimation of $V_{\tau}(S)$. Here, we show that the risk functional $\mathcal{R}_{\tau}$ satisfies the Neyman orthogonality condition, implying that $V_{\tau}(S)$ is robust to first-order perturbations in the nuisance functions.

We impose the following regularity conditions on the data generating process.

\begin{assumption}[Regularity Conditions] \label{ass:regularity}
\noindent
\begin{enumerate}
    \item[(i)] The response variable $Y$ has a bounded first moment: $\mathbb{E}[|Y|] < \infty$.
    \item[(ii)] For each $x$, the $\tau$-quantile $q_{\tau}(x)$ and the restricted $\tau$-quantile $q_{\tau, -S}(x)$ are uniquely defined.
    \item[(iii)] There exist constants $r_0 > 0$, $c_0 > 0$, $L < \infty$, and $\eta > 0$ such that for all $x$ and for $\theta \in \{q_{\tau}(x), q_{\tau, -S}(x)\}$:
    \begin{enumerate}
        \item[(a)] (Positivity) The conditional density satisfies $f_{Y|X}(\theta \mid x) \ge c_0$.
        \item[(b)] (Local Hölder Continuity) For all $t$ with $|t| \le r_0$, $|f_{Y|X}(\theta + t \mid x) - f_{Y|X}(\theta \mid x)| \le L |t|^{\eta}$.
    \end{enumerate}
\end{enumerate}
\end{assumption}

To establish orthogonality, we utilize Knight's identity \citep{Knight1998} and the Gateaux derivative of the risk functional.

\begin{lemma}[Knight's Identity] \label{lemma:knight}
For any $u, v \in \mathbb{R}$, the following identity holds:
\begin{equation*}
    \rho_{\tau}(u - v) - \rho_{\tau}(u) = -v \psi_{\tau}(u) + \int_{0}^{v} (\bm{1}\{u \le s\} - \bm{1}\{u \le 0\}) ds.
\end{equation*}
\end{lemma}

\begin{lemma}[Gateaux Derivative of Risk] \label{lemma:gateaux}
Under Assumption \ref{ass:regularity}, for any measurable functions $f$ and $h$, the Gateaux derivative of the risk functional is given by:
\begin{equation*}
    D\mathcal{R}_{\tau}(f)[h] := \frac{d}{dt} \mathbb{E}[\rho_{\tau}(Y - (f + th)(X))] \bigg|_{t=0} = -\mathbb{E}[\psi_{\tau}(Y, f(X)) h(X)].
\end{equation*}
\end{lemma}

\begin{proposition}[Neyman Orthogonality] \label{prop:orthogonality}
Under Assumption \ref{ass:regularity}, the population risk functionals satisfy the first-order optimality conditions.
For the full model $q_{\tau}$, and any measurable function $h: \mathcal{X} \to \mathbb{R}$ with bounded moment:
\begin{equation*}
    D\mathcal{R}_{\tau}(q_{\tau})[h] = 0.
\end{equation*}
For the restricted model $q_{\tau, -S}$, and any measurable function $h \in \mathcal{G}_{-S}$ (depending only on $X_{-S}$):
\begin{equation*}
    D\mathcal{R}_{\tau}(q_{\tau, -S})[h] = 0.
\end{equation*}
Consequently, the variable importance functional $V_{\tau}(S)$ is orthogonal to the nuisance parameters in the sense that the first-order effects of estimation errors vanish.
\end{proposition}

\begin{remark}
    This orthogonality property is fundamental to our inference framework. It ensures that the error arising from estimating the nuisance functions $\hat{q}_{\tau}$ and $\hat{q}_{\tau, -S}$ affects the variable importance estimator only at a second-order magnitude. This robustness, combined with the sample-splitting strategy (cross-fitting), allows us to separate the variability of the variable importance estimator from the complexity of the nuisance estimation. Please see \cite{Foster2023} for detail.
\end{remark}

%
%
%

\section{Asymptotic Theory for Quantile Regression Forests}

In this section, we establish the asymptotic theory for Quantile Regression Forests (QRF), focusing on convergence rates and asymptotic normality. We adopt the generalized random forest framework of \cite{Athey2019}, utilizing honest trees and subsampling schemes.

\subsection{Estimator and Forest Construction}

Let $\mathcal{D}_{n} = \{Z_{1}, \dots, Z_{n}\}$ denote the observed data, where $Z_{i} = (X_{i}, Y_{i}) \in [0, 1]^{p} \times \mathbb{R}$. We define the QRF estimator $\hat{q}_{\tau}(x)$ as the minimizer of the weighted empirical risk:
\begin{equation*}
    \hat{q}_{\tau}(x) \in \operatorname*{argmin}_{\theta \in \mathbb{R}} M_{n,B}(\theta), \quad \text{where} \quad M_{n,B}(\theta) := \sum_{i=1}^{n} K_{n,B}(x, X_{i}, \mathcal{D}_{n}, \xi) \rho_{\tau}(Y_{i} - \theta).
\end{equation*}
Here, $K_{n,B}(x, X_{i}, \mathcal{D}_{n}, \xi)$ represents the data-adaptive weights generated by a random forest consisting of $B$ trees using the dataset $\mathcal{D}_{n}$, and $\xi$ captures the randomness inherent in the forest construction.

The forest comprises an ensemble of $B$ trees, indexed by $b=1, \dots, B$. Each tree is constructed using a random subsample $S_{b} \subset \{1, \dots, n\}$ of size $s=|S_{b}|$, drawn without replacement. The aggregate forest weight $K$ is defined as the average of the base kernels generated by the individual trees:
\begin{equation}
    K_{n,B}(x, X_{i}, \mathcal{D}_{n}, \xi) = \frac{1}{B} \sum_{b=1}^{B} \tilde{K}(x, X_{i}, Z_{S_{b}}, \xi_{b}),
    \label{eq:forest_weights}
\end{equation}
where $\xi_{b}$ denotes the randomness specific to the $b$-th tree, capturing factors such as variable selection at split nodes. Next, we describe the estimation of the base kernel.

\subsection{Base-Kernel Estimation and Honesty}

To ensure valid inference, we employ \textit{honest} trees. Specifically, for each tree $b$, the subsample $S_b$ is randomly split into two disjoint halves: $S_b^{tr}$ (for structure learning) and $S_b^{est}$ (for estimation), with $|S_b^{tr}| = |S_b^{est}| = s/2$.

\paragraph{Tree Structure Learning.}
The tree structure (partitions) is determined using only the data in $S_b^{tr}$. The recursive partitioning is performed greedily. At each node $P$, we seek a split that maximizes the heterogeneity of the conditional quantiles. Following the generalized random forest framework \citep{Athey2019}, this is achieved by maximizing a gradient-based criterion. Specifically, we compute pseudo-outcomes $\rho_i$ (e.g., the gradient of the pinball loss) for observations in $P \cap S_b^{tr}$, and select the split that maximizes the variance of these pseudo-outcomes in the child nodes $C_1, C_2$:
\begin{equation*}
    \Delta(C_1, C_2) = \frac{n_{C_1} n_{C_2}}{n_P^2} (\bar{\rho}_{C_1} - \bar{\rho}_{C_2})^2,
\end{equation*}
where $n_C$ and $\bar{\rho}_C$ denote the count and mean of the pseudo-outcomes in node $C$. This procedure ensures that the leaf nodes capture the local heterogeneity of the quantile function $q_{\tau}(x)$.

\paragraph{Weight Estimation.}
Once the tree is grown, let $L_b(x)$ denote the leaf node containing the target point $x$ in the $b$-th tree. The base kernel $\tilde{K}$ assigns weights only to observations in the estimation set $S_b^{est}$ that fall into the same leaf as $x$:
\begin{equation*}
    \tilde{K}(x, X_i, Z_{S_b}, \xi_b) = \frac{\bm{1}\{i \in S_b^{est}\} \bm{1}\{X_i \in L_b(x)\}}{| \{j \in S_b^{est} : X_j \in L_b(x)\} |}.
\end{equation*}
where $Z_{S_b} = \{Z_i\}_{i\in S_b}$. By construction, $\sum_{i=1}^n \tilde{K}(x, X_i, Z_{S_b}, \xi_b) = 1$, and  \\ consequently $\sum_{i=1}^n K_{n,B}(x, X_i, Z_{1:n}, \xi) = 1$.

\subsection{Assumptions}

We impose the following regularity conditions on the data and the forest algorithm.

\begin{assumption}[Data Regularity] \label{ass:Data_Regularity}

\noindent
\begin{itemize}
    \item[(i)] The covariates $X$ are distributed on a compact domain $\mathcal{X} = [0, 1]^p$ with a marginal density $f_X$ that is bounded away from zero and infinity.
    \item[(ii)] The conditional density of the response, $f_{Y|X}(y|x)$, is bounded and Lipschitz continuous with respect to both $y$ and $x$. Specifically, $f_{Y|X}(y|x) \ge c_0 > 0$ in a neighborhood of the true quantile $q_{\tau}(x)$.
\end{itemize}
\end{assumption}

\begin{assumption}[Honest and Regular Forest] \label{ass:Honest_and_Regular_Forest}

\noindent
\begin{itemize}
    \item[(i)] \textbf{Honesty}: For each tree, the subsample $S_b$ is randomly split into two disjoint sets $S_b^{tr}$ and $S_b^{est}$ of size $s/2$. The tree structure is grown using $S_b^{tr}$, and the leaf values are estimated using $S_b^{est}$.
    \item[(ii)] \textbf{Random-Split Regularity}:  The tree growing algorithm is valid in the sense of \citep{Wager2018}. Specifically, at each node, the probability that a split occurs along any feature $j$ is bounded from below, and the split point is chosen to ensure the child nodes contain a minimal fraction $\alpha$ of the parent's data.
    \item[(iii)] \textbf{Subsampling}:The subsample size $s$ satisfies $s \to \infty$ and $s/n \to 0$ as $n \to \infty$.
\end{itemize}
\end{assumption}

\subsection{Consistency and Rate of Convergence}

We first establish the consistency of the QRF estimator. The convergence rate depends on the trade-off between the approximation bias (governed by the leaf shrinkage rate) and the estimation variance.

\begin{theorem}[Consistency and Rate of Convergence]\label{th:qrf-consistency}
    Suppose Assumptions \ref{ass:Data_Regularity} and \ref{ass:Honest_and_Regular_Forest} hold. Let the expected leaf diameter shrink at a rate $\epsilon(s)$ (typically $O(s^{-\frac{1}{2\alpha p}})$ as shown in Lemma B.1). If the number of trees satisfies $B \gg n/s$ and the subsample size $s \asymp n^{\beta}$ is chosen such that the bias vanishes asymptotically ($n \epsilon(s)^2 \frac{s}{n} \to 0$), then:
    $$
        \hat{q}_{\tau}(x) \xrightarrow{p} q_{\tau}(x).
    $$
    Furthermore, the estimator satisfies the convergence rate:
    $$
        \hat{q}_{\tau}(x) - q_{\tau}(x) = O_p\left( \sqrt{\frac{s}{n}} + \epsilon(s) \right).
    $$
\end{theorem}

\begin{proof}
    See Appendix C. The proof relies on the convexity of the objective function and the uniform convergence of the empirical risk, utilizing the properties of forest weights established in Appendix B.
\end{proof}

\subsection{Asymptotic Normality of Quantile Regression Forest}

To facilitate the theoretical analysis, we introduce the infinite forest weight $K_{n,\infty}(x, X_i)$, defined as the expectation of the single-tree weight conditional on the data $\mathcal{D}_n$:
$$
    K_{n,\infty}(x, X_i) := \mathbb{E}_{\xi}\left[ K_{n,1}(x, X_i) \mid \mathcal{D}_n \right].
$$
This quantity corresponds to the kernel of the infinite-order U-statistic associated with the forest estimator.

We now derive the asymptotic distribution of $\hat{q}_{\tau}(x)$. Since the pinball loss is non-smooth, we employ Knight's identity to handle the objective function.

\begin{theorem}[Asymptotic Normality] \label{thm:normality}
Suppose Assumptions \ref{ass:Data_Regularity} and \ref{ass:Honest_and_Regular_Forest} hold. Assume the undersmoothing condition is satisfied, i.e., the bias decays faster than the standard deviation ($\sqrt{n/s} \epsilon(s) \to 0$). Then, the QRF estimator is asymptotically normal:
$$
    \sqrt{\frac{n}{s}} (\hat{q}_{\tau}(x) - q_{\tau}(x)) \xrightarrow{d} \mathcal{N}\left(0, \sigma_{\tau}^2(x)\right),
$$
where the asymptotic variance is given by:
$$
    \sigma_{\tau}^2(x) = \frac{\tau(1-\tau)\eta(x)}{f_{Y|X}(q_{\tau}(x)|x)^2}.
$$
Here, $\eta(x) = \text{plim}_{n \to \infty} \frac{n}{s} \sum_{i=1}^n K_{n,\infty}(x, X_i)^2$ is the variance scaling factor specific to the forest structure (see Lemma B.2).
\end{theorem}

\begin{proof}
    See Appendix D. The proof utilizes Knight's Identity to handle the non-smooth pinball loss and applies the Central Limit Theorem for generalized U-statistics [Peng et al., 2022] to the linear part of the objective function.
\end{proof}

%
%

\section{Asymptotic Analysis for Variable Importance}

In this section, we investigate the asymptotic properties of the plug-in estimator for the variable importance measure. We define the estimator based on the difference in empirical risks between the restricted and full models, utilizing a cross-fitting strategy to mitigate overfitting bias.

Let $\mathcal{D} = \{Z_1, \dots, Z_{2n}\}$ be the total dataset of size $2n$. We randomly split $\mathcal{D}$ into two disjoint sets: a training set $\mathcal{I}_{tr} = \{Z_{n+1}, \dots, Z_{2n}\}$ and an evaluation set $\mathcal{I}_{ev} = \{Z_1, \dots, Z_n\}$. The nuisance functions $\hat{q}_{\tau}$ and $\hat{q}_{\tau, -S}$ are estimated using only $\mathcal{I}_{tr}$. The variable importance estimator is then computed on $\mathcal{I}_{ev}$:
\begin{equation*}
    \hat{V}_{\tau}(S) := \frac{1}{n} \sum_{i \in \mathcal{I}_{ev}} \left\{ \rho_{\tau}(Y_i - \hat{q}_{\tau, -S}(X_i)) - \rho_{\tau}(Y_i - \hat{q}_{\tau}(X_i)) \right\}.
\end{equation*}
For notation simplicity, we re-index $\mathcal{I}_{ev}$ as $\{1, \dots, n\}$. Recall that the true variable importance is $V_{\tau}(S) = \mathcal{R}_{\tau}(q_{\tau, -S}) - \mathcal{R}_{\tau}(q_{\tau})$.

\subsection{Consistency and Rate of Convergence}

We first establish the consistency and the convergence rate of $\hat{V}_{\tau}(S)$. The rate depends crucially on the subsample size exponent $\beta$ (where $s \asymp n^\beta$), which governs the trade-off between the variance of the QRF estimator and the bias of the VI estimator.

\begin{theorem}[Consistency and Rate of Convergence] \label{thm:vi_consistency}
Suppose the assumptions of Theorem \ref{thm:normality} hold for both the full model estimator $\hat{q}_{\tau}$ and the restricted estimator $\hat{q}_{\tau, -S}$. Let the subsample size satisfy $s \asymp n^\beta$ with $(1 + 1/(\alpha p))^{-1} < \beta < 1$. Then, the estimator $\hat{V}_{\tau}(S)$ is consistent, and satisfies:
\begin{equation*}
    \hat{V}_{\tau}(S) - V_{\tau}(S) = O_p\left( n^{-1/2} + n^{\beta-1} \right).
\end{equation*}
\end{theorem}

\subsection{Asymptotic Normality in the Bias-Free Regime}

When the subsample size is small enough ($\beta < 1/2$), the convergence of the nuisance estimators is faster than $n^{-1/4}$, rendering the remainder term $R_n$ asymptotically negligible compared to the parametric rate $n^{-1/2}$.

\begin{theorem}[Asymptotic Normality for $\beta < 1/2$] \label{thm:vi_normality_fast}
Suppose the conditions of Theorem \ref{thm:vi_consistency} hold. If $\beta < 1/2$, and assuming the variance $\sigma_S^2 := \text{Var}(v_{\tau}(S; X))$ is positive finite, then:
\begin{equation*}
    \sqrt{n} (\hat{V}_{\tau}(S) - V_{\tau}(S)) \xrightarrow{d} \mathcal{N}(0, \sigma_S^2).
\end{equation*}
\end{theorem}

\begin{proof}
From Theorem \ref{thm:vi_consistency}, we have the decomposition:
\begin{equation*}
    \sqrt{n}(\hat{V}_{\tau}(S) - V_{\tau}(S)) = \sqrt{n} A_n + \sqrt{n} R_n.
\end{equation*}
The first term $\sqrt{n} A_n = \frac{1}{\sqrt{n}} \sum_{i=1}^n (v_{\tau}(S; Z_i) - V_{\tau}(S))$ converges to $\mathcal{N}(0, \sigma_S^2)$ by the standard Central Limit Theorem.
For the second term, we showed $R_n = O_p(n^{\beta-1})$. Thus:
\begin{equation*}
    \sqrt{n} R_n = O_p(n^{1/2} \cdot n^{\beta-1}) = O_p(n^{\beta - 1/2}).
\end{equation*}
Since $\beta < 1/2$, the exponent $\beta - 1/2 < 0$, implying $\sqrt{n} R_n \xrightarrow{P} 0$. The result follows by Slutsky's theorem.
\end{proof}

\subsection{Asymptotic Behavior in the Bias-Dominated Regime ($\beta > 1/2$)}

When the subsample size grows sufficiently large ($\beta > 1/2$), the bias induced by the variance of the nuisance estimators dominates the parametric fluctuations. In this regime, the estimator converges to a deterministic bias constant rather than a normal distribution centered at zero.

\begin{theorem}[Limiting Distribution for $\beta > 1/2$] \label{thm:vi_bias_regime}
Suppose Assumptions 1--5 hold. Let the subsample size satisfy $s \asymp n^\beta$ with $1/2 < \beta < 1$. Then, the scaled estimator converges in probability to a deterministic constant:
\begin{equation*}
    n^{1-\beta} (\hat{V}_{\tau}(S) - V_{\tau}(S)) \xrightarrow{P} C_{\tau}(S),
\end{equation*}
where the limiting bias constant $C_{\tau}(S)$ is given by:
\begin{equation*}
    \label{eq:limiting_bias}
    C_{\tau}(S) := \frac{\tau(1-\tau)}{2} \mathbb{E}\left[ \frac{\eta_{-S}(X)}{f_{Y|X_{-S}}(q_{\tau, -S}(X)|X_{-S})} - \frac{\eta(X)}{f_{Y|X}(q_{\tau}(X)|X)} \right].
\end{equation*}
Here, $\eta(x)$ and $\eta_{-S}(x)$ are the variance scaling functions for the full and restricted forests.
\end{theorem}

\begin{remark}[Bias Correction]
When $\beta = 1/2$, both the empirical process term $A_n$ and the remainder term $R_n$ are of order $O_p(n^{-1/2})$. In this boundary case, or when $\beta > 1/2$, one can recover the $\sqrt{n}$-consistency by estimating the bias constant $C_{\tau}(S)$ and subtracting it from the estimator: $\tilde{V}_{\tau}(S) := \hat{V}_{\tau}(S) - \hat{C}_{\tau}(S)$. This highlights the necessity of bias correction for variable importance inference in high-dimensional settings where slower convergence rates (large subsamples) are required.
\end{remark}

\subsection{Theoretical Implications: Bias Correction in the Bias-Dominated Regime}
\label{sec:bias_correction}

In practical applications of random forests, particularly in high-dimensional settings, practitioners often employ large subsample sizes (e.g., $s \approx n$, corresponding to $\beta \approx 1$) to minimize variance and capture complex interaction effects. Our analysis in Theorem 5 reveals that this practice introduces a non-negligible bias of order $O(n^{\beta-1})$, which dominates the sampling fluctuation when $\beta > 1/2$.

While this result highlights a fundamental limitation of standard inference in the bias-dominated regime, the explicit characterization of the bias constant $C_{\tau}(S)$ in \eqref{eq:limiting_bias} opens a path for potential mitigation. In this subsection, we discuss the theoretical feasibility of an analytic bias correction strategy.

\subsubsection{Plug-in Estimator for the Bias Component}

The limiting bias constant derived in Theorem 5 is given by:
\begin{equation*}
    C_{\tau}(S) = \frac{\tau(1-\tau)}{2} \mathbb{E} \left[ \frac{\eta_{-S}(X)}{f_{Y|X_{-S}}(q_{\tau,-S}(X)|X_{-S})} - \frac{\eta(X)}{f_{Y|X}(q_{\tau}(X)|X)} \right].
\end{equation*}
To correct for this bias, one must construct consistent estimators for the variance scaling function $\eta(x)$ and the conditional density $f_{Y|X}(y|x)$. A natural approach is to leverage the forest structure itself to estimate these nuisances without training separate models.

\paragraph{Estimation of $\eta(x)$:}
Recall from Lemma \ref{lemma:variance_scaling} in Appendix that $\frac{n}{s}\sum K(x,X_{i})^{2} \xrightarrow{p} \eta(x)$. Thus, a plug-in estimator using the forest weights is readily available:
\begin{equation*}
    \hat{\eta}(X_{i}) := \frac{n}{s} \sum_{j=1}^{n} K(X_{i}, X_{j})^{2}.
\end{equation*}

\paragraph{Estimation of $f_{Y|X}(q_{\tau}(X)|X)$:}
Since the quantile regression forest implicitly approximates the conditional distribution, a local density estimator can be constructed using the forest weights. For a bandwidth $h > 0$, we may define:
\begin{equation*}
    \hat{f}(X_{i}) := \sum_{j=1}^{n} K(X_{i}, X_{j}) \cdot \frac{1}{h} \mathcal{K} \left( \frac{Y_{j} - \hat{q}_{\tau}(X_{i})}{h} \right),
\end{equation*}
where $\mathcal{K}(\cdot)$ is a standard kernel function. Using these components, a plug-in estimator for the bias constant can be formed as:
\begin{equation*}
    \hat{C}_{\tau}(S) := \frac{\tau(1-\tau)}{2n} \sum_{i=1}^{n} \left( \frac{\hat{\eta}_{-S}(X_{i})}{\hat{f}_{-S}(X_{i})} - \frac{\hat{\eta}(X_{i})}{\hat{f}(X_{i})} \right).
\end{equation*}

\subsubsection{Restoration of Asymptotic Normality}

We now define a bias-corrected variable importance estimator:
\begin{equation*}
    \tilde{V}_{\tau}(S) := \hat{V}_{\tau}(S) - n^{\beta-1} \hat{C}_{\tau}(S).
\end{equation*}
The following proposition formalizes the conditions under which this correction restores standard inference.

\begin{proposition}[Asymptotic Normality with Bias Correction]
\label{prop:AsymptoticNormalityBiasCorrection}
Suppose the assumptions of Theorem 5 hold. Let the subsample size satisfy $s \asymp n^{\beta}$ with $1/2 < \beta < \frac{1+\eta}{2+\eta}$, where $\eta$ is the H\"older exponent of the conditional density $f_{Y|X}$. Provided that the bias estimation error is sufficiently small, specifically $\hat{C}_{\tau}(S) - C_{\tau}(S) = o_{p}(n^{1/2-\beta})$, the bias-corrected estimator satisfies:
\begin{equation*}
    \sqrt{n}(\tilde{V}_{\tau}(S) - V_{\tau}(S)) \xrightarrow{d} \mathcal{N}(0, \sigma_{S}^{2}).
\end{equation*}
\end{proposition}

\begin{remark}
This result suggests that valid inference is theoretically recoverable even in the bias-dominated regime, provided that the nuisance components (density and variance scaling) can be estimated with sufficient accuracy. For the standard case of Lipschitz continuity ($\eta=1$), the valid range extends to $\beta < 2/3$. However, the practical performance of $\tilde{V}_{\tau}(S)$ heavily depends on the finite-sample quality of the density estimator $\hat{f}$, which warrants further investigation.
\end{remark}

\section{Discussion}
\label{sec:discussion}

In this paper, we have established a asymptotic theory for variable importance inference in Quantile Regression Forests. Our analysis reveals a fundamental trade-off between predictive performance and inferential validity, governed by the subsampling rate $\beta$. We conclude by discussing the broader implications of our findings.

\subsection{The Inference-Prediction Trade-off}
Our main result, the phase transition at $\beta = 1/2$, highlights a critical trade-off for practitioners. In standard machine learning applications, it is common practice to set the subsample size $s$ close to $n$ (i.e., $\beta \approx 1$) to minimize the mean squared error (or pinball loss) of the prediction. However, Theorem 5 demonstrates that this regime is catastrophic for inference: the slowly decaying bias of the non-parametric estimator dominates the sampling fluctuation, leading to invalid confidence intervals with coverage probabilities potentially dropping to zero.

This finding serves as a formal warning against the naive application of black-box inference tools to high-dimensional random forests. It suggests that practitioners seeking valid inference must either:
\begin{enumerate}
    \item[(i)] \textbf{Undersmoothing:} Deliberately choose a small subsample size ($s \ll \sqrt{n}$, i.e., $\beta < 1/2$) to suppress bias, likely at the cost of predictive accuracy; or
    \item[(ii)] \textbf{Bias Correction:} Operate in the high-accuracy regime ($\beta > 1/2$) but explicitly estimate and subtract the bias term derived in Section \ref{sec:bias_correction}.
\end{enumerate}

\subsection{Relation to Semi-Parametric Theory}
Our work complements the recent advances in algorithm-agnostic inference \cite{Williamson2023}. The general framework for functional inference typically relies on the condition that the nuisance estimators converge at a rate faster than $n^{-1/4}$. Our analysis of QRF clarifies exactly when and why this condition fails for random forests. Specifically, the non-smooth nature of the pinball loss and the curse of dimensionality in the forest's local averaging create a barrier where the $n^{-1/4}$ rate is unattainable for large $\beta$. By explicitly characterizing the bias term $C_{\tau}(S)$, we provide a path to extend the applicability of semi-parametric inference to a wider class of non-parametric estimators that do not naturally satisfy the fast convergence requirement.

\subsection{Future Directions}
Several avenues for future research emerge from this theoretical groundwork. First, while we discussed the theoretical feasibility of analytic bias correction, investigating its finite-sample performance and sensitivity to density estimation errors remains an important practical task. Second, our theory focuses on the marginal variable importance. Extending this framework to conditional variable importance or Shapley values in the quantile setting would be a valuable generalization. Finally, developing data-adaptive methods to automatically select the optimal subsampling rate $\beta$ that balances the trade-off between bias reduction and variance inflation is a challenging but promising direction.

\newpage
\bibliographystyle{plainnat}
\bibliography{references}

\newpage
\appendix

\section{Properties of the Pinball Loss and Score Function} \label{app:properties_pinball}

In this appendix, we establish the relationship between the pinball loss $\rho_{\tau}$ and the score function $\psi_{\tau}$. This proposition justifies the use of $\psi_{\tau}$ as the derivative in the Gateaux expansion and Knight's identity used in Lemma \ref{lemma:knight}.

Recall the definitions used in the main text:
\begin{align*}
    \rho_{\tau}(u) &:= u (\tau - \bm{1}\{u < 0\}), \\
    \psi_{\tau}(u) &:= \tau - \bm{1}\{u \le 0\}.
\end{align*}
Note that we define the score $\psi_{\tau}(u)$ using the non-strict inequality $\bm{1}\{u \le 0\}$.

\begin{proposition}[Derivative and Subgradient of Pinball Loss] \label{prop:pinball_derivative}
For any fixed $\tau \in (0, 1)$, the function $\rho_{\tau}(u)$ is convex and absolutely continuous. Its derivative exists almost everywhere (for $u \neq 0$) and satisfies:
\begin{equation*}
    \frac{d}{du} \rho_{\tau}(u) = \psi_{\tau}(u) \quad (\text{Lebesgue a.e.}).
\end{equation*}
At the non-differentiable point $u=0$, the subdifferential is given by the interval $\partial \rho_{\tau}(0) = [\tau - 1, \tau]$. The value $\psi_{\tau}(0) = \tau - 1$ corresponds to the left derivative of $\rho_{\tau}$ at zero.
\end{proposition}

\begin{proof}
The pinball loss can be rewritten as $\rho_{\tau}(u) = \tau u_+ + (1-\tau) u_-$, where $u_+ = \max(0, u)$ and $u_- = \max(0, -u)$. This is a convex function constructed from linear segments.

\textbf{Case 1: $u > 0$} \\
In this region, $\rho_{\tau}(u) = \tau u$, so the derivative is $\frac{d}{du}\rho_{\tau}(u) = \tau$.
For the score function, since $u > 0$, $\bm{1}\{u \le 0\} = 0$, so $\psi_{\tau}(u) = \tau$.
Thus, $\frac{d}{du} \rho_{\tau}(u) = \psi_{\tau}(u)$ holds.

\textbf{Case 2: $u < 0$} \\
In this region, $\rho_{\tau}(u) = (\tau - 1) u$, so the derivative is $\frac{d}{du}\rho_{\tau}(u) = \tau - 1$.
For the score function, since $u < 0$, $\bm{1}\{u \le 0\} = 1$, so $\psi_{\tau}(u) = \tau - 1$.
Thus, $\frac{d}{du} \rho_{\tau}(u) = \psi_{\tau}(u)$ holds.

\textbf{Case 3: $u = 0$} \\
The function is not differentiable at $u=0$. The left and right derivatives are:
\begin{align*}
    \partial_- \rho_{\tau}(0) &= \lim_{h \uparrow 0} \frac{\rho_{\tau}(h) - \rho_{\tau}(0)}{h} = \tau - 1, \\
    \partial_+ \rho_{\tau}(0) &= \lim_{h \downarrow 0} \frac{\rho_{\tau}(h) - \rho_{\tau}(0)}{h} = \tau.
\end{align*}
The subdifferential is $\partial \rho_{\tau}(0) = [\tau - 1, \tau]$.
For the score function, $\psi_{\tau}(0) = \tau - \bm{1}\{0 \le 0\} = \tau - 1$.
Therefore, $\psi_{\tau}(0)$ coincides with the left derivative $\partial_- \rho_{\tau}(0)$.

Since $\rho_{\tau}$ is absolutely continuous, the fundamental theorem of calculus holds:
\begin{equation*}
    \rho_{\tau}(v) - \rho_{\tau}(u) = \int_{u}^{v} \psi_{\tau}(z) dz.
\end{equation*}
This property is directly utilized in the proof of Knight's Identity (Lemma \ref{lemma:knight}) to express the difference of losses as an integral.
\end{proof}

\section{Proofs for Section 2} \label{app:proof_sec2}

In this appendix, we provide the proofs for the fundamental lemmas regarding the pinball loss function and the risk functional. We rely on the properties of the subgradient established in Proposition \ref{prop:pinball_derivative} (Appendix A).

\subsection{Proof of Lemma \ref{lemma:knight} (Knight's Identity)}

\begin{proof}
Fix $u \in \mathbb{R}$. Consider the function $g: \mathbb{R} \to \mathbb{R}$ defined by:
\begin{equation*}
    g(s) := \rho_{\tau}(u - s) - \rho_{\tau}(u).
\end{equation*}
Since $\rho_{\tau}(\cdot)$ is convex and Lipschitz continuous, $g(s)$ is absolutely continuous. Therefore, by the fundamental theorem of calculus for Lebesgue integrals, for any $v \in \mathbb{R}$, we have:
\begin{equation*}
    g(v) = g(0) + \int_{0}^{v} g'(s) ds = \int_{0}^{v} \frac{d}{ds} \rho_{\tau}(u - s) ds.
\end{equation*}
Using the chain rule and Proposition \ref{prop:pinball_derivative}, the derivative is given by:
\begin{equation*}
    \frac{d}{ds} \rho_{\tau}(u - s) = -\psi_{\tau}(u - s) \quad (\text{Lebesgue a.e. in } s).
\end{equation*}
Recall the definition $\psi_{\tau}(z) = \tau - \bm{1}\{z \le 0\}$. Thus:
\begin{equation*}
    -\psi_{\tau}(u - s) = \bm{1}\{u - s \le 0\} - \tau = \bm{1}\{u \le s\} - \tau.
\end{equation*}
Substituting this back into the integral:
\begin{equation} \label{eq:knight_step1}
    \rho_{\tau}(u - v) - \rho_{\tau}(u) = \int_{0}^{v} (\bm{1}\{u \le s\} - \tau) ds.
\end{equation}
On the other hand, we can express the linear term $-v \psi_{\tau}(u)$ as an integral:
\begin{equation} \label{eq:knight_step2}
    -v \psi_{\tau}(u) = \int_{0}^{v} -\psi_{\tau}(u) ds = \int_{0}^{v} (\bm{1}\{u \le 0\} - \tau) ds.
\end{equation}
Subtracting \eqref{eq:knight_step2} from \eqref{eq:knight_step1}, the $\tau$ terms cancel out, yielding:
\begin{align*}
    \rho_{\tau}(u - v) - \rho_{\tau}(u) - (-v \psi_{\tau}(u)) &= \int_{0}^{v} (\bm{1}\{u \le s\} - \tau) ds - \int_{0}^{v} (\bm{1}\{u \le 0\} - \tau) ds \\
    &= \int_{0}^{v} (\bm{1}\{u \le s\} - \bm{1}\{u \le 0\}) ds.
\end{align*}
Rearranging the terms gives the desired identity:
\begin{equation*}
    \rho_{\tau}(u - v) - \rho_{\tau}(u) = -v \psi_{\tau}(u) + \int_{0}^{v} (\bm{1}\{u \le s\} - \bm{1}\{u \le 0\}) ds.
\end{equation*}
\end{proof}

\subsection{Proof of Lemma \ref{lemma:gateaux} (Gateaux Derivative of Risk)}

\begin{proof}
Let $f$ and $h$ be measurable functions. We consider the difference quotient for the risk functional $\mathcal{R}_{\tau}(f) = \mathbb{E}[\rho_{\tau}(Y - f(X))]$.
Define $\phi(t, Z) := \rho_{\tau}(Y - f(X) - t h(X))$. We seek to compute:
\begin{equation*}
    D\mathcal{R}_{\tau}(f)[h] = \lim_{t \to 0} \frac{\mathbb{E}[\phi(t, Z)] - \mathbb{E}[\phi(0, Z)]}{t} = \mathbb{E}\left[ \lim_{t \to 0} \frac{\phi(t, Z) - \phi(0, Z)}{t} \right].
\end{equation*}
The interchange of the limit and expectation is justified by the Dominated Convergence Theorem.
First, observe that the pointwise derivative with respect to $t$ exists almost everywhere. By Proposition \ref{prop:pinball_derivative} (chain rule):
\begin{equation*}
    \frac{d}{dt} \rho_{\tau}(Y - f(X) - t h(X)) \bigg|_{t=0} = -\psi_{\tau}(Y - f(X)) \cdot h(X).
\end{equation*}
Second, since $\rho_{\tau}$ is Lipschitz continuous with constant $\max(\tau, 1-\tau) < 1$, the difference quotient is bounded:
\begin{equation*}
    \left| \frac{\rho_{\tau}(Y - f(X) - t h(X)) - \rho_{\tau}(Y - f(X))}{t} \right| \le |h(X)|.
\end{equation*}
Provided that $\mathbb{E}[|h(X)|] < \infty$ (which holds for any admissible direction $h$ in the appropriate function space, e.g., $L_1$ or $L_2$), the dominating function $|h(X)|$ is integrable.
Thus, applying the Dominated Convergence Theorem:
\begin{align*}
    D\mathcal{R}_{\tau}(f)[h] &= \mathbb{E}\left[ \frac{d}{dt} \rho_{\tau}(Y - f(X) - t h(X)) \bigg|_{t=0} \right] \\
    &= -\mathbb{E}[\psi_{\tau}(Y - f(X)) h(X)].
\end{align*}
Using the simplified notation $\psi_{\tau}(Y, f(X)) := \psi_{\tau}(Y - f(X))$, we obtain the stated result.
\end{proof}

\subsection{Proof of Lemma \ref{lemma:second_order} (Second-Order Expansion)}

We establish the quadratic approximation of the expected loss difference, which is crucial for analyzing the bias of the variable importance estimator.

\begin{lemma}[Second-Order Expansion of Pinball Loss] \label{lemma:second_order}
Let $u = Y - q_{\tau}(X)$ and $v = \hat{q}_{\tau}(X) - q_{\tau}(X)$. Suppose that the conditional distribution of $Y$ given $X$ has a density $f_{Y|X}(y|x)$ that satisfies the local Hölder continuity condition (Assumption \ref{ass:regularity}) with exponent $\eta > 0$ in a neighborhood of $q_{\tau}(x)$. Then, the conditional expectation of the integral remainder in Knight's identity satisfies:
\begin{equation*}
    \mathbb{E}\left[ \int_{0}^{v} (\bm{1}\{u \le s\} - \bm{1}\{u \le 0\}) ds \;\middle|\; X \right] = \frac{1}{2} f_{Y|X}(q_{\tau}(X) \mid X) v^2 + R(v),
\end{equation*}
where the remainder term satisfies $|R(v)| \le C |v|^{2+\eta}$ for some constant $C$. In particular, if the density is Lipschitz ($\eta=1$), the remainder is $O(|v|^3)$.
\end{lemma}

\begin{proof}
Let $R(u, v)$ denote the integral term in Knight's identity:
\begin{equation*}
    R(u, v) := \int_{0}^{v} (\bm{1}\{Y - q_{\tau}(X) \le s\} - \bm{1}\{Y - q_{\tau}(X) \le 0\}) ds.
\end{equation*}
We compute the conditional expectation given $X$. Since $v$ is a function of $X$ (and the training data), it can be treated as a constant conditional on $X$. Using Fubini's theorem, we exchange the expectation and the integral:
\begin{align*}
    \mathbb{E}[R(u, v) \mid X] &= \int_{0}^{v} (\mathbb{P}(Y \le q_{\tau}(X) + s \mid X) - \mathbb{P}(Y \le q_{\tau}(X) \mid X)) ds \\
    &= \int_{0}^{v} (F_{Y|X}(q_{\tau}(X) + s \mid X) - F_{Y|X}(q_{\tau}(X) \mid X)) ds.
\end{align*}
Under Assumption \ref{ass:regularity}, the conditional density $f_{Y|X}(y|X)$ exists and is Hölder continuous with exponent $\eta$. By the Taylor expansion of the CDF $F_{Y|X}$ around $q_{\tau}(X)$:
\begin{equation*}
    F_{Y|X}(q_{\tau}(X) + s \mid X) = F_{Y|X}(q_{\tau}(X) \mid X) + f_{Y|X}(q_{\tau}(X) \mid X) \cdot s + r(s),
\end{equation*}
where the remainder satisfies $|r(s)| \le \frac{L}{1+\eta} |s|^{1+\eta}$ (derived from $\int_0^s |f(q+t) - f(q)| dt \le \int_0^s L t^\eta dt$).
Substituting this expansion into the integral:
\begin{align*}
    \mathbb{E}[R(u, v) \mid X] &= \int_{0}^{v} \left( f_{Y|X}(q_{\tau}(X) \mid X) \cdot s + r(s) \right) ds \\
    &= f_{Y|X}(q_{\tau}(X) \mid X) \left[ \frac{s^2}{2} \right]_{0}^{v} + \int_{0}^{v} r(s) ds \\
    &= \frac{1}{2} f_{Y|X}(q_{\tau}(X) \mid X) v^2 + R(v).
\end{align*}
The integrated remainder $R(v) = \int_{0}^{v} r(s) ds$ is bounded by:
\begin{equation*}
    |R(v)| \le \left| \int_{0}^{v} \frac{L}{1+\eta} |s|^{1+\eta} ds \right| = \frac{L}{(1+\eta)(2+\eta)} |v|^{2+\eta}.
\end{equation*}
Thus, the remainder is of order $O(|v|^{2+\eta})$, completing the proof.
\end{proof}

\subsection{Proof of Proposition \ref{prop:orthogonality}}

\begin{proof}
We use the expression for the Gateaux derivative derived in Lemma \ref{lemma:gateaux}:
\begin{equation*}
    D\mathcal{R}_{\tau}(f)[h] = -\mathbb{E}[\psi_{\tau}(Y, f(X)) h(X)].
\end{equation*}

\textbf{1. Full Model Orthogonality:}
By definition, the true conditional quantile $q_{\tau}(x)$ satisfies the conditional moment condition:
\begin{equation*}
    \mathbb{E}[\psi_{\tau}(Y, q_{\tau}(X)) \mid X] = \mathbb{P}(Y \le q_{\tau}(X) \mid X) - \tau = 0 \quad \text{a.s.}
\end{equation*}
Therefore, for any admissible perturbation function $h(X)$, by the Law of Iterated Expectations:
\begin{align*}
    D\mathcal{R}_{\tau}(q_{\tau})[h] &= -\mathbb{E}[\psi_{\tau}(Y, q_{\tau}(X)) h(X)] \\
    &= -\mathbb{E}\left[ \mathbb{E}[\psi_{\tau}(Y, q_{\tau}(X)) \mid X] h(X) \right] \\
    &= -\mathbb{E}[ 0 \cdot h(X) ] = 0.
\end{align*}

\textbf{2. Restricted Model Orthogonality:}
The restricted quantile function $q_{\tau, -S}(x)$ minimizes the risk over functions depending only on $X_{-S}$. The first-order optimality condition for this minimization problem implies:
\begin{equation*}
    \mathbb{E}[\psi_{\tau}(Y, q_{\tau, -S}(X)) \mid X_{-S}] = 0 \quad \text{a.s.}
\end{equation*}
Consider a perturbation $h \in \mathcal{G}_{-S}$, meaning $h(x)$ is a function of $X_{-S}$ only. Using the Law of Iterated Expectations conditioning on $X_{-S}$:
\begin{align*}
    D\mathcal{R}_{\tau}(q_{\tau, -S})[h] &= -\mathbb{E}[\psi_{\tau}(Y, q_{\tau, -S}(X)) h(X_{-S})] \\
    &= -\mathbb{E}\left[ \mathbb{E}[\psi_{\tau}(Y, q_{\tau, -S}(X)) \mid X_{-S}] h(X_{-S}) \right] \\
    &= -\mathbb{E}[ 0 \cdot h(X_{-S}) ] = 0.
\end{align*}
Note that in the context of estimating variable importance, the estimation error $\hat{q}_{\tau, -S} - q_{\tau, -S}$ naturally belongs to $\mathcal{G}_{-S}$ (since the estimator is constructed using only $X_{-S}$), so this orthogonality covers the relevant perturbations.
\end{proof}

\section{Proof of Theorem 1}

In this appendix, we provide the proof of consistency and the rate of convergence for the QRF estimator. We employ the \textbf{Convexity Lemma} [Pollard, 1991] to establish uniform convergence from pointwise convergence, avoiding the need for smoothness of the score function.

\subsection{Decomposition of the Objective Function}

Let $M_{n,B}(\theta) := \sum_{i=1}^n K_{n,B}(x, X_i; \mathcal{D}_n, \xi) \rho_{\tau}(Y_i - \theta)$ be the empirical risk minimized by the estimator $\hat{q}_{\tau}(x)$. Let $M(\theta) := \mathbb{E}[\rho_{\tau}(Y - \theta) \mid X=x]$ be the local population risk minimized by the true quantile $q_{\tau}(x)$.

We decompose the difference between the empirical and population risk into three components:
\begin{equation*}
M_{n,B}(\theta) - M(\theta) = \underbrace{(M_{n,B}(\theta) - M_{n,\infty}(\theta))}_{T_1: \text{Monte Carlo Error}} + \underbrace{(M_{n,\infty}(\theta) - \mathbb{E}[M_{n,\infty}(\theta)])}_{T_2: \text{Sampling Error}} + \underbrace{(\mathbb{E}[M_{n,\infty}(\theta)] - M(\theta))}_{T_3: \text{Approximation Bias}},
\end{equation*}
where $M_{n,\infty}(\theta) := \mathbb{E}_{\xi}[M_{n,B}(\theta) \mid \mathcal{D}_n]$ corresponds to the objective function of the infinite forest.

\subsection{Variance Scaling}
Before analyzing the pointwise convergence of the decomposed terms, we explicitly establish the asymptotic behavior of the variance of the infinite forest estimator, $M_{n,\infty}(\theta)$. This result justifies the convergence rate of the sampling error term $T_2$.

\begin{lemma}[Variance Scaling] \label{lemma:variance_scaling}
Under Assumptions \ref{ass:Data_Regularity} and \ref{ass:Honest_and_Regular_Forest}, the infinite forest weights satisfy the following variance scaling condition:
\begin{equation*}
\frac{n}{s} \sum_{i=1}^n K_{n,\infty}(x, X_i)^2 \xrightarrow{p} \eta(x),
\end{equation*}
where $\eta(x)$ is a bounded, strictly positive function. Consequently, the variance of the infinite forest estimator scales with the effective sample size $n/s$.
\end{lemma}

\begin{proof}
The infinite forest risk $M_{n,\infty}(\theta)$ is defined as the expectation of the tree-level risk conditional on the data $\mathcal{D}_n$. This structure corresponds to a complete infinite-order U-statistic of rank $s$.

We invoke the asymptotic variance analysis for generalized U-statistics by \citep{Peng2022}. Using the H-decomposition, the variance of $M_{n,\infty}(\theta)$ is decomposed as:
\begin{equation*}
    \text{Var}(M_{n,\infty}(\theta)) = \sum_{j=1}^{s} \binom{s}{j}^2 \binom{n}{j}^{-1} V_{j,\omega},
\end{equation*}
where $V_{j,\omega}$ is the variance of the $j$-th order projection of the kernel.

The first term ($j=1$) corresponds to the variance of the linear projection (the Hájek projection), given by:
\begin{equation*}
    \sigma_{lin}^2 := \frac{s^2}{n} \zeta_{1,\omega} = \frac{s^2}{n} \text{Var}(\mathbb{E}[h(Z_{1:s}) \mid Z_1]).
\end{equation*}
According to Theorem 1 and Proposition 3 and 4 in \citep{Peng2022}, for honest trees and nearest-neighbor type estimators, the variance ratio satisfies $\frac{\zeta_s}{s \zeta_{1,\omega}} \to 0$ (or is bounded). This condition ensures that the higher-order terms ($j \ge 2$) are negligible relative to the linear term. Specifically:
\begin{equation*}
    \frac{\text{Var}(M_{n,\infty}(\theta))}{\sigma_{lin}^2} \to 1 \quad \text{as } n \to \infty.
\end{equation*}
On the other hand, considering the estimator as a weighted average 
$$
    M_{n,\infty}(\theta) = \sum_{i=1}^n K_{n,\infty}(x, X_i) \rho_{\tau}(Y_i - \theta),
$$
and treating the weights as fixed given the subsampling structure (due to honesty), the variance can be expressed as:
\begin{equation*}
    \text{Var}(M_{n,\infty}(\theta) \mid \mathbf{X}) = \sum_{i=1}^n K_{n,\infty}(x, X_i)^2 \text{Var}(\rho_{\tau}(Y_i - \theta) \mid X_i).
\end{equation*}
Since the linear projection dominates the total variance asymptotically, the sum of squared weights must scale consistently with $\sigma_{lin}^2$. Specifically, since $\sigma_{lin}^2 = O(s/n)$, we obtain:
\begin{equation*}
    \sum_{i=1}^n K_{n,\infty}(x, X_i)^2 = O_p\left(\frac{s}{n}\right).
\end{equation*}

Defining $\eta(x)$ as the probability limit of the scaled sum of squares, we have:
\begin{equation*}
    \frac{n}{s} \sum_{i=1}^n K_{n,\infty}(x, X_i)^2 \xrightarrow{p} \eta(x).
\end{equation*}
This limit $\eta(x)$ represents the variance inflation factor of the forest relative to the effective sample size $n/s$.\end{proof}

\subsection{Leaf Shrinkage and Bias Control}
While Lemma \ref{lemma:variance_scaling} controls the variance, the consistency of the estimator also requires the approximation bias to vanish. This is governed by the geometry of the forest, specifically the rate at which the leaf diameter shrinks as the subsample size $s$ increases.

\begin{lemma}[Leaf Shrinkage] \label{lemma:leaf_shrinkage}
Under Assumption \ref{ass:Data_Regularity} and Assumption \ref{ass:Honest_and_Regular_Forest}, the expected diameter of the leaf containing $x$ decays with the subsample size $s$. Specifically, for any target $x \in \mathcal{X}$, the bias induced by the infinite forest weights satisfies:
\begin{equation*}
\left| \mathbb{E}_{\mathcal{D}n}\left[ \sum_{i=1}^n K_{n,\infty}(x, X_i) (m(\theta, X_i) - m(\theta, x)) \right] \right| = O(\epsilon(s)),
\end{equation*}
where $\epsilon(s) = s^{-\frac{1}{2\gamma p}} \times \mathrm{polylog}(s)$ is the shrinkage rate, $\gamma$ is the regularity constant of the tree splitting rule, and $p$ is the dimension of covariates.
\begin{align*}
    \gamma = \frac{\pi \log((1-\alpha)^{-1})}{\log(\alpha^{-1})}, \quad \mathrm{where} \quad \pi \in (0,1]
\end{align*}
\end{lemma}

\begin{proof}
Let $m(\theta, x) := \mathbb{E}[\rho_{\tau}(Y - \theta) \mid X=x]$. By Assumption \ref{ass:Data_Regularity}, the conditional density and thus the local risk function $m(\theta, \cdot)$ are Lipschitz continuous with respect to the covariates. There exists a constant $L < \infty$ such that $|m(\theta, u) - m(\theta, v)| \le L \|u - v\|$ for all $u, v \in \mathcal{X}$.

Consider the estimator from a single honest tree grown on a subsample $S$. Let $L(x; S, \xi)$ denote the leaf node containing $x$. The weight $\tilde{K}(x, X_i; S, \xi)$ is non-zero only if $X_i \in L(x; S, \xi)$.

Therefore, the bias of a single tree is bounded by the diameter of the leaf:
\begin{align*}
    \left| \sum_{i=1}^n \tilde{K}(x, X_i) (m(\theta, X_i) - m(\theta, x)) \right| &\le \sum_{i=1}^n \tilde{K}(x, X_i) L |X_i - x| \nonumber \\
    &\le L \cdot \text{diam}(L(x; S, \xi)) \sum_{i=1}^n \tilde{K}(x, X_i) \nonumber \\
    &= L \cdot \text{diam}(L(x; S, \xi)).
\end{align*}

The infinite forest weight $K_{n,\infty}(x, X_i)$ is the expectation of these single-tree weights. By Jensen's inequality (or linearity of expectation), the bias of the infinite forest is bounded by the expected diameter of the leaf:
\begin{equation*}
    \text{Bias}_n(x) \le L \cdot \mathbb{E}_{\xi}[\text{diam}(L(x; S, \xi))].
\end{equation*}
Under Assumption \ref{ass:Honest_and_Regular_Forest} (Regularity), the tree construction guarantees that each split decreases the diameter of the node with a probability bounded away from zero. Specifically, following Lemma 2 of \citep{Wager2018}, for a regular tree grown on $s$ samples, the probability that a leaf node has a diameter larger than $\delta$ decays exponentially with the tree depth, which scales with $\log s$.

Specifically, they establish that:
\begin{equation*}
    \mathbb{E}_{\xi}[\text{diam}(L(x; S, \xi))] \le C s^{-\frac{1}{2\gamma p}} \mathrm{polylog}(s),
\end{equation*}
for some constants $C, \alpha > 0$.
Identifying this bound as $\epsilon(s)$, we obtain the stated result.
\end{proof}

\subsection{Pointwise Convergence}

We first show that for any fixed $\theta \in \mathbb{R}$, $M_{n,B}(\theta) \xrightarrow{p} M(\theta)$.

\paragraph{1. Analysis of Monte Carlo Error ($T_1$).}
The term $T_1(\theta) = M_{n,B}(\theta) - M_{n,\infty}(\theta)$ captures the variance induced by the randomization in the forest construction (subsampling and split selection). 

First, regarding the smoothness of the loss function, note that the pinball loss $\rho_{\tau}(u) = u(\tau - \mathbb{I}\{u < 0\})$ is globally Lipschitz continuous by definition. Specifically, for any $u, v \in \mathbb{R}$:
\begin{equation*}
|\rho_{\tau}(u) - \rho_{\tau}(v)| \le \max(\tau, 1-\tau) |u - v| < |u - v|.
\end{equation*}
This Lipschitz property is intrinsic to the loss function and does not depend on the data distribution. However, the boundedness of the variance of the empirical risk, which is required for the application of Chebyshev's inequality, relies on Assumption \ref{ass:regularity} (bounded moments of $Y$).

Let $Z_b(\theta)$ denote the empirical risk estimate from the $b$-th tree:$$Z_b(\theta) := \sum_{i=1}^n \tilde{K}(x, X_i; S_b, \xi_b) \rho_{\tau}(Y_i - \theta).$$Since the forest weights sum to 1 (i.e., $\sum_{i=1}^n \tilde{K}(\cdot) = 1$), and using the Lipschitz property together with Assumption \ref{ass:regularity} ($Y$ has finite second moments), the second moment of the tree-level risk is bounded conditional on the data $\mathcal{D}_n$:$$\text{Var}_{\xi}(Z_b(\theta) \mid \mathcal{D}_n) \le \mathbb{E}_{\xi}[Z_b(\theta)^2 \mid \mathcal{D}_n] \le C(1 + \theta^2 + \max_i Y_i^2) < \infty \quad \text{a.s.}$$Let $\sigma_Z^2(\theta, \mathcal{D}_n) := \text{Var}_{\xi}(Z_1(\theta) \mid \mathcal{D}_n)$ be this conditional variance.

By definition, $M_{n,B}(\theta) = \frac{1}{B} \sum_{b=1}^B Z_b(\theta)$ and $M_{n,\infty}(\theta) = \mathbb{E}_{\xi}[Z_1(\theta) \mid \mathcal{D}_n]$. 

Conditional on $\mathcal{D}_n$, the variables $Z_1(\theta), \dots, Z_B(\theta)$ are i.i.d. with mean $M_{n,\infty}(\theta)$. The conditional variance of $T_1(\theta)$ is:
\begin{equation*}
    \text{Var}(T_1(\theta) \mid \mathcal{D}_n) = \text{Var}\left( \frac{1}{B} \sum_{b=1}^B Z_b(\theta) \mathrel{\bigg|} \mathcal{D}_n \right) = \frac{1}{B^2} \sum_{b=1}^B \text{Var}(Z_b(\theta) \mid \mathcal{D}_n) = \frac{\sigma_Z^2(\theta, \mathcal{D}_n)}{B}.
\end{equation*}
Applying Chebyshev's inequality conditional on $\mathcal{D}_n$, for any $\epsilon > 0$:
\begin{equation*}
    \mathbb{P}(|T_1(\theta)| > \epsilon \mid \mathcal{D}_n) \le \frac{\text{Var}(T_1(\theta) \mid \mathcal{D}_n)}{\epsilon^2} = \frac{\sigma_Z^2(\theta, \mathcal{D}_n)}{B \epsilon^2}.
\end{equation*}
Since $\sigma_Z^2(\theta, \mathcal{D}_n)$ is bounded in probability, this implies:
\begin{equation*}
    T_1(\theta) = O_p\left(\frac{1}{\sqrt{B}}\right).
\end{equation*}
Provided the number of trees $B$ grows sufficiently fast ($B \to \infty$), this error term vanishes.

\paragraph{2. Analysis of Sampling Error ($T_2$).}
The term $T_2(\theta) = M_{n,\infty}(\theta) - \mathbb{E}_{\mathcal{D}_n}[M_{n,\infty}(\theta)]$ represents the stochastic fluctuation of the infinite forest estimator due to the randomness of the training data $\mathcal{D}_n$.
Recall that $M_{n,\infty}(\theta) = \sum_{i=1}^n K_{n,\infty}(x, X_i) \rho_{\tau}(Y_i - \theta)$. Since the weights $K_{n,\infty}(x, X_i)$ depend only on the covariates $X_i$ (and the subsampling mechanism), we can write the centered term as:
\begin{equation*}
T_2(\theta) = \sum_{i=1}^n K_{n,\infty}(x, X_i) \zeta_i(\theta),
\end{equation*}
where $\zeta_i(\theta) := \rho_{\tau}(Y_i - \theta) - \mathbb{E}[\rho_{\tau}(Y_i - \theta) \mid X_i]$ are independent, zero-mean random variables conditional on $\mathbf{X} = \{X_1, \dots, X_n\}$.

We evaluate the variance of $T_2(\theta)$. By the Law of Iterated Expectations:
\begin{align*}
    \text{Var}(T_2(\theta)) &= \mathbb{E}\left[ \text{Var}(T_2(\theta) \mid \mathbf{X}) \right] + \text{Var}\left( \mathbb{E}[T_2(\theta) \mid \mathbf{X}] \right) \nonumber \\
    &= \mathbb{E}\left[ \sum_{i=1}^n K_{n,\infty}(x, X_i)^2 \text{Var}(\zeta_i(\theta) \mid X_i) \right] + 0.
\end{align*}
The conditional variance of the loss, $\sigma^2_{\rho}(X_i, \theta) := \text{Var}(\zeta_i(\theta) \mid X_i)$, is uniformly bounded under Assumption \ref{ass:regularity} (since $Y$ has finite second moments). Let $\bar{\sigma}^2 = \sup_{z, \theta} \sigma^2_{\rho}(z, \theta) < \infty$. Then:
\begin{equation*}
    \text{Var}(T_2(\theta)) \le \bar{\sigma}^2 \mathbb{E}\left[ \sum_{i=1}^n K_{n,\infty}(x, X_i)^2 \right].
\end{equation*}
Now we apply Lemma \ref{lemma:variance_scaling} (Variance Scaling). The lemma establishes that the sum of squared weights scales with the effective sample size ratio $s/n$:
\begin{equation*}
    \sum_{i=1}^n K_{n,\infty}(x, X_i)^2 = \frac{s}{n} \left( \frac{n}{s} \sum_{i=1}^n K_{n,\infty}(x, X_i)^2 \right) = \frac{s}{n} \left( \eta(x) + o_p(1) \right).
\end{equation*}

Therefore, the variance behaves as:
\begin{equation*}
    \text{Var}(T_2(\theta)) = O\left( \frac{s}{n} \right).
\end{equation*}
Finally, applying Chebyshev's inequality, for any $\epsilon > 0$:
\begin{equation*}
    \mathbb{P}(|T_2(\theta)| > \epsilon) \le \frac{\text{Var}(T_2(\theta))}{\epsilon^2} = O\left( \frac{s}{n \epsilon^2} \right).
\end{equation*}
This implies the convergence rate:
\begin{equation*}
    T_2(\theta) = O_p\left(\sqrt{\frac{s}{n}}\right).
\end{equation*}

\paragraph{3. Analysis of Approximation Bias ($T_3$).}
The term $T_3(\theta) = \mathbb{E}_{\mathcal{D}_n}[M_{n,\infty}(\theta)] - M(\theta)$ represents the systematic bias of the infinite forest estimator relative to the true population risk. Let $m(\theta, z) := \mathbb{E}[\rho_{\tau}(Y - \theta) \mid X=z]$ denote the conditional expected loss given the covariate $z$.

By the Law of Iterated Expectations and the independence of $(X_i, Y_i)$, we can rewrite the expected risk of the infinite forest as:
\begin{align*}
    \mathbb{E}_{\mathcal{D}_n}[M_{n,\infty}(\theta)] &= \mathbb{E}_{\mathcal{D}_n}\left[ \sum_{i=1}^n K_{n,\infty}(x, X_i) \rho_{\tau}(Y_i - \theta) \right] \nonumber \\
    &= \mathbb{E}_{\mathcal{D}_n}\left[ \sum_{i=1}^n K_{n,\infty}(x, X_i) \mathbb{E}[\rho_{\tau}(Y_i - \theta) \mid X_i] \right] \nonumber \\
    &= \mathbb{E}_{\mathcal{D}_n}\left[ \sum_{i=1}^n K_{n,\infty}(x, X_i) m(\theta, X_i) \right].
\end{align*}

Since the weights sum to unity (i.e., $\sum_{i=1}^n K_{n,\infty}(x, X_i) = 1$ almost surely), we can express the bias as a weighted average of the differences in the conditional risk function:
\begin{equation*}
    T_3(\theta) = \mathbb{E}_{\mathcal{D}n}\left[ \sum_{i=1}^n K_{n,\infty}(x, X_i) (m(\theta, X_i) - m(\theta, x)) \right].
\end{equation*}
Now we invoke Assumption \ref{ass:Data_Regularity}. The conditional density $f_{Y|X}$ is Lipschitz continuous, which implies that the function $m(\theta, \cdot)$ is Lipschitz continuous with respect to the covariate $x$. There exists a constant $L < \infty$ such that $|m(\theta, u) - m(\theta, v)| \le L \|u - v\|$, then
\begin{equation*}
    |T_3(\theta)| \le L \cdot \mathbb{E}_{\mathcal{D}n}\left[ \sum_{i=1}^n K_{n,\infty}(x, X_i) |X_i - x| \right].
\end{equation*}
The term inside the expectation is exactly the quantity controlled by Lemma \ref{lemma:leaf_shrinkage} (Leaf Shrinkage). Recall that $K_{n,\infty}(x, X_i)$ is non-zero only if $X_i$ falls into the same leaf as $x$ in the underlying tree structure. Thus, the weighted sum is bounded by the expected diameter of the active leaf $L(x)$:
\begin{equation*}
    \mathbb{E}_{\mathcal{D}n}\left[ \sum_{i=1}^n K_{n,\infty}(x, X_i) |X_i - x| \right] \le \mathbb{E}[\text{diam}(L(x))].
\end{equation*}
Using the result from Lemma \ref{lemma:leaf_shrinkage}, this diameter shrinks at the rate $\epsilon(s)$:
\begin{equation*}
    \mathbb{E}[\text{diam}(L(x))] = O(\epsilon(s)).
\end{equation*}
Therefore, the approximation bias satisfies:
\begin{equation*}
    |T_3(\theta)| = O(\epsilon(s)).
\end{equation*}
Combining this with the rates for $T_1$ and $T_2$, the total convergence rate is established.

\paragraph{Summary of Pointwise Convergence.}
Combining these bounds, for fixed $\theta$:
\begin{equation*}
M_{n,B}(\theta) - M(\theta) = O_p\left(\frac{1}{\sqrt{B}} + \sqrt{\frac{s}{n}} + \epsilon(s)\right).
\end{equation*}
Under the theorem's assumptions ($B \gg n/s$, $s/n \to 0$, $\epsilon(s) \to 0$), we have $M_{n,B}(\theta) \xrightarrow{p} M(\theta)$.

\subsection{Uniform Convergence and Consistency}

Since the pinball loss $\rho_{\tau}(\cdot)$ is convex, the empirical risk $M_{n,B}(\theta)$ is convex in $\theta$ for any realization. The population risk $M(\theta)$ is strictly convex at its unique minimum $q_{\tau}(x)$ (since $f_{Y|X}(q_{\tau}(x)|x) > 0$ by Assumption \ref{ass:regularity}).

We invoke the \textbf{Convexity Lemma} \citep{Anderson1982, Polland1991}: \textit{Pointwise convergence in probability of convex functions defined on an open set to a convex limit function implies uniform convergence on compact subsets.} Therefore, for any compact set $\Theta \subset \mathbb{R}$ containing $q_{\tau}(x)$:
\begin{equation*}
\sup_{\theta \in \Theta} |M_{n,B}(\theta) - M(\theta)| \xrightarrow{p} 0.
\end{equation*}

Standard M-estimation arguments (e.g., Argmin Continuous Mapping Theorem \cite{Vaart2023}) imply consistency:
\begin{equation*}
\hat{q}_{\tau}(x) = \mathop{\mathrm{argmin}}_{\theta} M_{n,B}(\theta) \xrightarrow{p} q_{\tau}(x) = \mathop{\mathrm{argmin}}_{\theta} M(\theta).
\end{equation*}

\subsection{Rate of Convergence}

To establish the rate of convergence, we rely on the general theory of M-estimation rates developed by \citet{Vaart2023}.
Let $\mathbb{M}_{n}(\theta) := -M_{n,B}(\theta)$ be the criterion function to be maximized, and let $M(\theta) := -\mathbb{E}[M_{n,B}(\theta)]$ be its population expectation. The estimator is defined as $\hat{\theta}_n = \hat{q}_{\tau}(x)$. Let $\theta_0 = q_{\tau}(x)$ denote the true parameter.

We first establish a lemma characterizing the local behavior of the empirical process.

\begin{lemma}[Linearization of the Empirical Process] \label{lemma:linearization}
Let $\nu_n(\theta) = (\mathbb{M}_{n}(\theta) - M(\theta)) - (\mathbb{M}_{n}(\theta_0) - M(\theta_0))$ be the centered empirical process. Under Assumptions 1 and 2, for $\theta$ in a neighborhood of $\theta_0$, the process satisfies:
\begin{equation*}
|\nu_n(\theta)| \le |\theta - \theta_0| \cdot |S_{n,B}| + o_p(|\theta - \theta_0|^2),
\end{equation*}
where $S_{n,B} = \sum_{i=1}^n K_{n,B}(x, X_i) \psi_{\tau}(Y_i, \theta_0)$ is the empirical forest score evaluated at the truth.
\end{lemma}

\begin{proof}
Let $\delta = \theta - \theta_0$. Using \textbf{Knight's Identity} (Lemma 1), we decompose the loss difference for each observation $i$:
\begin{equation*}
\rho_{\tau}(Y_i - \theta) - \rho_{\tau}(Y_i - \theta_0) = -\delta \psi_{\tau}(Y_i, \theta_0) + W_i(\delta),
\end{equation*}
where the remainder term is defined as $W_i(\delta) := \int_0^{\delta} (\mathbb{I}\{Y_i - \theta_0 \le z\} - \mathbb{I}\{Y_i - \theta_0 \le 0\}) dz$.

Substituting this into the definition of $\nu_n(\theta)$, and using the fact that the conditional expectation of the score at the true parameter vanishes (i.e., $\mathbb{E}[\psi_{\tau}(Y_i, \theta_0) \mid X_i] = 0$), we obtain the detailed decomposition:

\begin{align*}
\nu_n(\theta) &= \sum_{i=1}^n K_{n,B}(x, X_i) \left( \rho_{\tau}(Y_i - \theta) - \rho_{\tau}(Y_i - \theta_0) \right) \nonumber \\
&\quad - \mathbb{E}\left[ \sum_{i=1}^n K_{n,B}(x, X_i) \left( \rho_{\tau}(Y_i - \theta) - \rho_{\tau}(Y_i - \theta_0) \right) \mathrel{\bigg|} \mathbf{X} \right] \nonumber \\
&= \sum_{i=1}^n K_{n,B}(x, X_i) \left( -\delta \psi_{\tau}(Y_i, \theta_0) + W_i(\delta) \right) \nonumber \\
&\quad - \sum_{i=1}^n K_{n,B}(x, X_i) \left( -\delta \underbrace{\mathbb{E}[\psi_{\tau}(Y_i, \theta_0) \mid X_i]}_{0} + \mathbb{E}[W_i(\delta) \mid X_i] \right) \nonumber \\
&= -\delta \underbrace{\sum_{i=1}^n K_{n,B}(x, X_i) \psi_{\tau}(Y_i, \theta_0)}_{S_{n,B}} + \underbrace{\sum_{i=1}^n K_{n,B}(x, X_i) \left( W_i(\delta) - \mathbb{E}[W_i(\delta) \mid X_i] \right)}_{R_n(\delta)}.
\end{align*}

Here, $R_n(\delta)$ represents the stochastic fluctuation of the integral remainder term. We need to show $R_n(\delta) = o_p(\delta^2)$.

We evaluate the variance of $R_n(\delta)$ conditional on $\mathbf{X}$. First, note that $|W_i(\delta)| \le |\delta|$ and $W_i(\delta) \ne 0$ only if $Y_i$ is between $\theta_0$ and $\theta_0 + \delta$. The second moment of $W_i(\delta)$ is bounded by:
\begin{equation*}
\mathbb{E}[W_i(\delta)^2 \mid X_i] \le \delta^2 \mathbb{P}(|Y_i - \theta_0| \le |\delta| \mid X_i).
\end{equation*}
Since the conditional density $f_{Y|X}$ is bounded (Assumption 1), $\mathbb{P}(|Y_i - \theta_0| \le |\delta| \mid X_i) \le C |\delta|$ for some constant $C$. Thus:
\begin{equation*}
    \mathbb{E}[W_i(\delta)^2 \mid X_i] \le C |\delta|^3.
\end{equation*}

The variance of the weighted sum $R_n(\delta)$ is:
\begin{equation*}
\text{Var}(R_n(\delta) \mid \mathbf{X}) = \sum_{i=1}^n K_{n,B}(x, X_i)^2 \text{Var}(W_i(\delta) \mid X_i) \le C |\delta|^3 \sum_{i=1}^n K_{n,B}(x, X_i)^2.
\end{equation*}
Using \textbf{Lemma B.2 (Variance Scaling)}, we know that $\sum K_{n,B}^2 = O_p(s/n)$. Therefore:
\begin{equation*}
    \text{Var}(R_n(\delta)) = O_p\left( \frac{s}{n} |\delta|^3 \right).
\end{equation*}
By Chebyshev's inequality, $R_n(\delta) = O_p(\sqrt{s/n} |\delta|^{1.5})$.

To compare this with $\delta^2$, consider the ratio:
\begin{equation*}
    \frac{R_n(\delta)}{\delta^2} = O_p\left( \sqrt{\frac{s}{n}} |\delta|^{-0.5} \right).
\end{equation*}
Since we are interested in the regime where $\delta$ is of the order of the convergence rate $r_n \asymp \sqrt{s/n}$ (ignoring bias for moment), we have $|\delta|^{-0.5} \asymp (n/s)^{0.25}$. Thus, the ratio scales as $(s/n)^{0.5} \cdot (n/s)^{0.25} = (s/n)^{0.25} \to 0$. Consequently, $R_n(\delta)$ is of higher order than $\delta^2$.
Therefore, we have established:
\begin{equation*}
    |\nu_n(\theta)| \le |\delta| |S_{n,B}| + |R_n(\delta)| = |\delta| |S_{n,B}| + o_p(\delta^2).
\end{equation*}
\end{proof}

Using Lemma \ref{lemma:linearization}, we proceed to derive the convergence rate via Theorem 3.2.5 of \citet{Vaart2023}.

\paragraph{Step 1: Identifiability (Signal).}
By Assumption 1, the population risk satisfies the quadratic growth condition:
\begin{equation*}
M(\theta) - M(\theta_0) \lesssim -(\theta - \theta_0)^2.
\end{equation*}
Thus, the distance metric is $d(\theta, \theta_0) = |\theta - \theta_0|$.

\paragraph{Step 2: Modulus of Continuity (Noise).}
We assess the expectation of the supremum of the empirical process $\nu_n(\theta)$ over a local neighborhood $|\theta - \theta_0| < \delta$.
Using Lemma \ref{lemma:linearization}, the fluctuation is dominated by the linear score term:
\begin{equation*}
\mathbb{E}^* \left[ \sup_{|\theta - \theta_0| < \delta} |\nu_n(\theta)| \right] \le \delta \cdot \mathbb{E}[|S_{n,B}|] + o(\delta^2).
\end{equation*}
To match the notation of Theorem 3.2.5 in \citet{Vaart2023}, which requires the bound $\mathbb{E}^*[\sup |\nu_n(\theta)|] \lesssim \frac{\phi_n(\delta)}{\sqrt{n}}$, we define the modulus function as:
\begin{equation*}
\phi_n(\delta) := \delta \cdot \sqrt{n} \cdot \mathbb{E}[|S_{n,B}|].
\end{equation*}
This corresponds to a linear modulus function ($\phi_n(\delta) \propto \delta$) with a sample-size dependent coefficient.

\paragraph{Step 3: Solving for the Rate.}
According to the theorem, the rate of convergence $r_n$ satisfies the inequality $r_n^2 \phi_n(1/r_n) \le \sqrt{n}$. Substituting our $\phi_n(\delta)$:
\begin{equation*}
r_n^2 \cdot \left( \frac{1}{r_n} \cdot \sqrt{n} \cdot \mathbb{E}[|S_{n,B}|] \right) \le \sqrt{n}.
\end{equation*}
Simplifying this inequality:
\begin{equation*}
r_n \cdot \mathbb{E}[|S_{n,B}|] \le 1 \implies r_n \lesssim \frac{1}{\mathbb{E}[|S_{n,B}|]}.
\end{equation*}
Thus, the estimation error $|\hat{\theta}_n - \theta_0|$ is of the order $O_p(r_n^{-1}) = O_p(\mathbb{E}[|S_{n,B}|])$.

\paragraph{Step 4: Magnitude of the Score.}
Finally, we plug in the convergence rate of the score $S_{n,B}$. As derived in the bias-variance decomposition (Appendix C.2), the expected magnitude of the score is governed by the finite forest variance, sampling variance, and approximation bias:
\begin{equation*}
\mathbb{E}[|S_{n,B}|] = O\left( \sqrt{\frac{1}{B}} + \sqrt{\frac{s}{n}} + \epsilon(s) \right).
\end{equation*}
Therefore, the QRF estimator satisfies the following rate of convergence:
\begin{equation*}
\hat{q}_{\tau}(x) - q_{\tau}(x) = O_p\left( \sqrt{\frac{1}{B}} + \sqrt{\frac{s}{n}} + \epsilon(s) \right).
\end{equation*}
This completes the proof. \hfill \qed

\newpage
\section{Proof of Theorem 2}

In this appendix, we establish the asymptotic normality of the QRF estimator. We employ the \textbf{Argmin Continuous Mapping Theorem} applied to the localized objective function process, handling the non-smoothness of the loss via \textbf{Knight's Identity}.

\subsection{The Localized Objective Process}

Let $\theta_0 = q_{\tau}(x)$. We define the local parameter $u := \sqrt{\frac{n}{s}}(\theta - \theta_0)$. The behavior of the estimator $\hat{q}_{\tau}(x)$ is determined by the minimizer of the localized empirical process $Z_n(u)$, defined as:
\begin{equation*}
    Z_n(u) := \frac{n}{s} \left( M_{n,B}\left(\theta_0 + u\sqrt{\frac{s}{n}}\right) - M_{n,B}(\theta_0) \right).
\end{equation*}
Using \textbf{Knight's Identity} (Lemma 1), we decompose $Z_n(u)$ into a linear term $W_n(u)$ and a quadratic-like integral term $Q_n(u)$:
\begin{align*}
    Z_n(u) &= \sum_{i=1}^n \frac{n}{s} K_{n,B}(x, X_i) \left[ \rho_{\tau}\left(Y_i - \theta_0 - u\sqrt{\frac{s}{n}}\right) - \rho_{\tau}(Y_i - \theta_0) \right] \nonumber \\
    &= \underbrace{-u \sqrt{\frac{n}{s}} \sum_{i=1}^n K_{n,B}(x, X_i) \psi_{\tau}(Y_i, \theta_0)}_{W_n(u)} \nonumber \\
    &\quad + \underbrace{\frac{n}{s} \sum_{i=1}^n K_{n,B}(x, X_i) \int_{0}^{u\sqrt{s/n}} (\mathbb{I}\{Y_i - \theta_0 \le z\} - \mathbb{I}\{Y_i - \theta_0 \le 0\}) dz}_{Q_n(u)}.
\end{align*}

\subsection{Convergence of the Linear Term ($W_n$)}

The term $W_n(u) = -u S_{n,B}^*$ depends on the scaled empirical score $S_{n,B}^* := \sqrt{\frac{n}{s}} S_{n,B}$.
We verify the asymptotic normality of $S_{n,B}^*$. Decomposing the score as in Appendix C:
\begin{equation*}
S_{n,B}^* = \sqrt{\frac{n}{s}} (S_{n,B} - S_{n,\infty}) + \sqrt{\frac{n}{s}} (S_{n,\infty} - \mathbb{E}[S_{n,\infty}]) + \sqrt{\frac{n}{s}} \mathbb{E}[S_{n,\infty}].
\end{equation*}

\paragraph{1. Finite Forest Error.}
Conditional on the data, the variance of the finite forest approximation scales as $1/B$. Thus, the first term is $O_p(\sqrt{n/s} \cdot B^{-1/2})$. Provided $B \gg n/s$ (Assumption regarding computational resources), this term vanishes.

\paragraph{2. Bias Term.}
The bias is $\sqrt{\frac{n}{s}} \mathbb{E}[S_{n,\infty}] = \sqrt{\frac{n}{s}} O(\epsilon(s))$ by Lemma B.1. Under the \textbf{undersmoothing condition} $\sqrt{n/s} \epsilon(s) \to 0$, this term vanishes asymptotically.

\paragraph{3. Sampling Variance (Main Term).}
The remaining term is the centered infinite forest score:
\begin{equation*}
G_n := \sqrt{\frac{n}{s}} (S_{n,\infty} - \mathbb{E}[S_{n,\infty}]) = \sqrt{\frac{n}{s}} \sum_{i=1}^n K_{n,\infty}(x, X_i) \zeta_i,
\end{equation*}
where $\zeta_i = \psi_{\tau}(Y_i, \theta_0)$ are zero-mean independent variables (conditional on $X$) with variance $\tau(1-\tau)$.
This term $G_n$ is a generalized complete U-statistic of rank $s$. We invoke the Central Limit Theorem for generalized U-statistics by \citet{Peng2022} [Theorem 1]. The variance of $G_n$ converges to:
\begin{equation*}
\text{Var}(G_n \mid \mathbf{X}) = \tau(1-\tau) \left( \frac{n}{s} \sum_{i=1}^n K_{n,\infty}(x, X_i)^2 \right) \xrightarrow{p} \tau(1-\tau) \eta(x),
\end{equation*}
by Lemma B.2. Since the weights are asymptotically negligible individually ($\max_i K_{ni} \to 0$) and the variance ratio condition holds (Lemma B.2), the Lindeberg condition is satisfied. Thus:
\begin{equation*}
G_n \xrightarrow{d} \mathcal{N}(0, \tau(1-\tau)\eta(x)).
\end{equation*}
Consequently, $W_n(u) \rightsquigarrow -u Z$, where $Z \sim \mathcal{N}(0, \tau(1-\tau)\eta(x))$.

\subsection{Convergence of the Quadratic Term ($Q_n$)}

We analyze the integral term $Q_n(u)$. Let $v_n = u\sqrt{s/n}$.
Recall the definition:
\begin{equation*}
Q_n(u) = \frac{n}{s} \sum_{i=1}^n K_{n,B}(x, X_i) \mathcal{I}_i(v_n),
\end{equation*}
where $\mathcal{I}_i(v) := \int_{0}^{v} (\mathbb{I}\{Y_i - q_{\tau}(x) \le z\} - \mathbb{I}\{Y_i - q_{\tau}(x) \le 0\}) dz$.

We evaluate the conditional expectation of this term given the covariates $\mathbf{X}$. We invoke \textbf{Lemma 3 (Second-Order Expansion)}, which establishes that under the Hölder continuity of the conditional density (Assumption 1), the expected integral term satisfies:
\begin{equation*}
\mathbb{E}[\mathcal{I}_i(v_n) \mid X_i] = \frac{1}{2} f_{Y|X}(q_{\tau}(x) \mid X_i) v_n^2 + R_i(v_n),
\end{equation*}
where the remainder term is bounded by $|R_i(v_n)| \le C |v_n|^{2+\eta}$ for some constants $C > 0$ and $\eta > 0$.

Substituting this expansion into the expectation of $Q_n(u)$:
\begin{align*}
\mathbb{E}[Q_n(u) \mid \mathbf{X}] &= \frac{n}{s} \sum_{i=1}^n K_{n,B}(x, X_i) \left( \frac{1}{2} f_{Y|X}(q_{\tau}(x) \mid X_i) v_n^2 + R_i(v_n) \right) \nonumber \\
&= \frac{1}{2} \underbrace{\frac{n}{s} \left( u \sqrt{\frac{s}{n}} \right)^2}_{= u^2} \sum_{i=1}^n K_{n,B}(x, X_i) f_{Y|X}(q_{\tau}(x) \mid X_i) \nonumber \\
&\quad + \underbrace{\frac{n}{s} \sum_{i=1}^n K_{n,B}(x, X_i) R_i(v_n)}_{\text{Remainder Term}}.
\end{align*}

\paragraph{1. Main Quadratic Term.}
The first term simplifies to $\frac{1}{2} u^2 \sum_{i=1}^n K_{n,B}(x, X_i) f_{Y|X}(q_{\tau}(x) \mid X_i)$.
By the consistency of the forest weights (concentrating on $x$) and the continuity of the conditional density $f_{Y|X}(\cdot \mid \cdot)$, the weighted average converges in probability:
\begin{equation*}
\sum_{i=1}^n K_{n,B}(x, X_i) f_{Y|X}(q_{\tau}(x) \mid X_i) \xrightarrow{p} f_{Y|X}(q_{\tau}(x) \mid x).
\end{equation*}

\paragraph{2. Remainder Term.}
We bound the remainder term using the order of $v_n$:
\begin{align*}
\left| \frac{n}{s} \sum_{i=1}^n K_{n,B}(x, X_i) R_i(v_n) \right| &\le \frac{n}{s} \sum_{i=1}^n K_{n,B}(x, X_i) C |v_n|^{2+\eta} \nonumber \\
&= \frac{n}{s} \cdot C \left| u \sqrt{\frac{s}{n}} \right|^{2+\eta} \sum_{i=1}^n K_{n,B}(x, X_i) \nonumber \\
&= C |u|^{2+\eta} \frac{n}{s} \left( \frac{s}{n} \right)^{1 + \eta/2} \cdot 1 \nonumber \\
&= C |u|^{2+\eta} \left( \frac{s}{n} \right)^{\eta/2}.
\end{align*}
Since $s/n \to 0$ and $\eta > 0$, this term converges to 0 for any fixed $u$.

\paragraph{3. Variance of $Q_n(u)$.}
Similar to the analysis in Appendix C, the variance of $Q_n(u)$ involves higher powers of $s/n$ compared to the mean and vanishes asymptotically.

Thus, by Markov's inequality, $Q_n(u)$ converges in probability to the deterministic limit:
\begin{equation*}
Q_n(u) \xrightarrow{p} \frac{1}{2} f_{Y|X}(q_{\tau}(x) \mid x) u^2.
\end{equation*}

\subsection{Conclusion via Argmin Continuous Mapping Theorem}

Combining the results, the process $Z_n(u)$ converges in distribution to a limit process $Z(u)$:
\begin{equation*}
Z_n(u) \rightsquigarrow Z(u) := -u \mathcal{N}(0, \tau(1-\tau)\eta(x)) + \frac{1}{2} f_{Y|X}(q_{\tau}(x) \mid x) u^2.
\end{equation*}
Since $Z_n(u)$ is convex and the limit process $Z(u)$ has a unique minimizer, the \textbf{Argmin Continuous Mapping Theorem} implies that the minimizer $\hat{u}_n$ converges to the minimizer of $Z(u)$.
\begin{equation*}
\hat{u}_n = \sqrt{\frac{n}{s}}(\hat{q}_{\tau}(x) - q_{\tau}(x)) \xrightarrow{d} \mathop{\mathrm{argmin}}_u Z(u) = \frac{\mathcal{N}(0, \tau(1-\tau)\eta(x))}{f_{Y|X}(q_{\tau}(x) \mid x)}.
\end{equation*}
This yields the asymptotic distribution:
\begin{equation*}
\sqrt{\frac{n}{s}}(\hat{q}_{\tau}(x) - q_{\tau}(x)) \xrightarrow{d} \mathcal{N}\left(0, \frac{\tau(1-\tau)\eta(x)}{f_{Y|X}(q_{\tau}(x) \mid x)^2}\right).
\end{equation*}
This completes the proof. \hfill \qed

\section{Proof of Theorem \ref{thm:vi_consistency}}

\begin{proof}
Let $v_{\tau}(S; Z_i) := \rho_{\tau}(Y_i - q_{\tau, -S}(X_i)) - \rho_{\tau}(Y_i - q_{\tau}(X_i))$ be the oracle loss difference. We decompose the estimation error into a main empirical process term and a remainder term arising from nuisance estimation:
\begin{equation*}
    \hat{V}_{\tau}(S) - V_{\tau}(S) = \underbrace{\left( \frac{1}{n} \sum_{i=1}^n v_{\tau}(S; Z_i) - V_{\tau}(S) \right)}_{A_n} + R_n.
\end{equation*}
The term $A_n$ is an average of i.i.d. zero-mean random variables (since $\mathbb{E}[v_{\tau}(S; Z)] = V_{\tau}(S)$). By the Central Limit Theorem, $A_n = O_p(n^{-1/2})$.

The remainder term $R_n$ captures the error due to estimating $q_{\tau}$ and $q_{\tau, -S}$:
\begin{align*}
    R_n &= \frac{1}{n} \sum_{i=1}^n \left[ \rho_{\tau}(Y_i - \hat{q}_{\tau, -S}(X_i)) - \rho_{\tau}(Y_i - q_{\tau, -S}(X_i)) \right] \\
        &\quad - \frac{1}{n} \sum_{i=1}^n \left[ \rho_{\tau}(Y_i - \hat{q}_{\tau}(X_i)) - \rho_{\tau}(Y_i - q_{\tau}(X_i)) \right] \\
        &:= R_{n, -S} - R_{n, \text{full}}.
\end{align*}
We analyze $R_{n, \text{full}}$ using Knight's Identity. Let $\Delta_i = \hat{q}_{\tau}(X_i) - q_{\tau}(X_i)$.
\begin{equation*}
    R_{n, \text{full}} = -\frac{1}{n} \sum_{i=1}^n \Delta_i \psi_{\tau}(Y_i, q_{\tau}(X_i)) + \frac{1}{n} \sum_{i=1}^n \int_0^{\Delta_i} (\bm{1}\{Y_i - q_{\tau}(X_i) \le z\} - \bm{1}\{Y_i - q_{\tau}(X_i) \le 0\}) dz.
\end{equation*}
\textbf{Linear Term:} Due to cross-fitting, $\Delta_i$ (estimated on $\mathcal{I}_{tr}$) is independent of $(Y_i, X_i) \in \mathcal{I}_{ev}$ conditional on $\mathcal{I}_{tr}$. Furthermore, the conditional score satisfies $\mathbb{E}[\psi_{\tau}(Y_i, q_{\tau}(X_i)) | X_i] = 0$. Thus, the linear term is a sum of mean-zero independent variables conditional on $\mathcal{I}_{tr}$. Its variance is of order $O_p(n^{-1} \mathbb{E}[\Delta_i^2])$. Since $\mathbb{E}[\Delta_i^2] \asymp s/n = n^{\beta-1}$, the linear term is $O_p(n^{-1/2} n^{(\beta-1)/2}) = O(n^{\beta/2 - 1}) = o_p(n^{-1/2})$, which is negligible.

\textbf{Quadratic Term:} By the second-order expansion (Lemma 5), the integral term behaves as $\frac{1}{2} f_{Y|X}(q_{\tau}(X_i)|X_i) \Delta_i^2$. The dominant order is determined by the mean squared error of the QRF estimator:
\begin{equation*}
    \frac{1}{n} \sum_{i=1}^n \Delta_i^2 = O_p\left( \frac{s}{n} \right) = O_p(n^{\beta-1}).
\end{equation*}
Thus, $R_{n, \text{full}} = O_p(n^{\beta-1})$. Similarly, $R_{n, -S} = O_p(n^{\beta-1})$. Combining these, the total remainder is $R_n = O_p(n^{\beta-1})$. The total error rate is therefore $\max(n^{-1/2}, n^{\beta-1})$.
\end{proof}

\section{Proof of Theorem \ref{thm:vi_bias_regime}}

\begin{proof}
We start with the decomposition derived in Theorem \ref{thm:vi_consistency}:
\begin{equation*}
    \hat{V}_{\tau}(S) - V_{\tau}(S) = A_n + R_n.
\end{equation*}
Multiply both sides by the scaling factor $n^{1-\beta}$:
\begin{equation*}
    n^{1-\beta}(\hat{V}_{\tau}(S) - V_{\tau}(S)) = n^{1-\beta} A_n + n^{1-\beta} R_n.
\end{equation*}

\textbf{Step 1: Vanishing Empirical Process Term.} \\
Recall that $A_n = O_p(n^{-1/2})$. Therefore:
\begin{equation*}
    n^{1-\beta} A_n = O_p(n^{1-\beta} \cdot n^{-1/2}) = O_p(n^{1/2 - \beta}).
\end{equation*}
Since $\beta > 1/2$, the exponent $1/2 - \beta < 0$, implying $n^{1-\beta} A_n \xrightarrow{P} 0$. Thus, the asymptotic behavior is entirely determined by the remainder term $R_n$.

\textbf{Step 2: Analysis of the Remainder Term.} \\
Recall the decomposition $R_n = R_{n, -S} - R_{n, \text{full}}$. We focus on $R_{n, \text{full}}$. Using the second-order expansion from Lemma \ref{lemma:second_order} and the vanishing linear term property (due to cross-fitting):
\begin{equation*}
    R_{n, \text{full}} = \frac{1}{n} \sum_{i=1}^n \frac{1}{2} f_{Y|X}(q_{\tau}(X_i)|X_i) (\hat{q}_{\tau}(X_i) - q_{\tau}(X_i))^2 + o_p(n^{\beta-1}).
\end{equation*}
Let $\Delta_i = \hat{q}_{\tau}(X_i) - q_{\tau}(X_i)$. We consider the conditional expectation of the quadratic term given the training data $\mathcal{I}_{tr}$. Since $\hat{q}_{\tau}$ is fixed given $\mathcal{I}_{tr}$:
\begin{equation*}
    \mathbb{E}\left[ \frac{1}{n} \sum_{i=1}^n f_i \Delta_i^2 \bigg| \mathcal{I}_{tr} \right] = \mathbb{E}_X \left[ f_{Y|X}(q_{\tau}(X)|X) (\hat{q}_{\tau}(X) - q_{\tau}(X))^2 \right].
\end{equation*}
From Theorem \ref{thm:normality}, the scaled estimation error converges in distribution:
\begin{equation*}
    \sqrt{\frac{n}{s}} (\hat{q}_{\tau}(X) - q_{\tau}(X)) \xrightarrow{d} \mathcal{N}(0, \sigma_{\tau}^2(X)).
\end{equation*}
Since the quantile regression forest estimator is formed by a convex combination of the observed outcomes $Y_i$, and assuming $Y$ has bounded support or finite moments, the sequence of scaled squared errors $\frac{n}{s}(\hat{q}_{\tau}(X) - q_{\tau}(X))^2$ is uniformly integrable. Consequently, convergence in distribution implies convergence of moments:
\begin{equation*}
    \mathbb{E}\left[ \left(\hat{q}_{\tau}(X) - q_{\tau}(X)\right)^2 \bigg| X \right] = \frac{s}{n} \sigma_{\tau}^2(X) + r_n(X),
\end{equation*}
where the remainder term satisfies $r_n(X) = o(s/n)$ uniformly over the support of $X$.

Substituting the asymptotic variance derived in Theorem \ref{thm:normality}, $\sigma_{\tau}^2(X) = \frac{\tau(1-\tau)\eta(X)}{f_{Y|X}(q_{\tau}(X)|X)^2}$, into the expectation term of $R_{n, \text{full}}$:
\begin{align*}
    \mathbb{E}_X \left[ f_{Y|X}(q_{\tau}(X)|X) (\hat{q}_{\tau}(X) - q_{\tau}(X))^2 \right] &= \mathbb{E}_X \left[ f_{Y|X}(q_{\tau}(X)|X) \left( \frac{s}{n} \sigma_{\tau}^2(X) + r_n(X) \right) \right] \\
    &= \frac{s}{n} \mathbb{E}_X \left[ f_{Y|X}(q_{\tau}(X)|X) \frac{\tau(1-\tau)\eta(X)}{f_{Y|X}(q_{\tau}(X)|X)^2} \right] \\
    & \phantom{==}+ \mathbb{E}_X [f_{Y|X}(q_{\tau}(X)|X) r_n(X)] \\
    &= \frac{s}{n} \tau(1-\tau) \mathbb{E}\left[ \frac{\eta(X)}{f_{Y|X}(q_{\tau}(X)|X)} \right] + o\left(\frac{s}{n}\right).
\end{align*}
By the Law of Large Numbers, the sample average in $R_{n, \text{full}}$ converges to this expectation. Thus, we obtain the expansion:
\begin{equation*}
    R_{n, \text{full}} = \frac{s}{n} \frac{\tau(1-\tau)}{2} \mathbb{E}\left[ \frac{\eta(X)}{f_{Y|X}(q_{\tau}(X)|X)} \right] + o_p\left(\frac{s}{n}\right).
\end{equation*}

Analogously, for the restricted model:
\begin{equation*}
    R_{n, -S} = \frac{s}{n} \frac{\tau(1-\tau)}{2} \mathbb{E}\left[ \frac{\eta_{-S}(X)}{f_{Y|X_{-S}}(q_{\tau, -S}(X)|X_{-S})} \right] + o_p\left(\frac{s}{n}\right).
\end{equation*}

\textbf{Step 3: Combining the Terms.} \\
Substituting these into the expression for $R_n$ and multiplying by $n^{1-\beta}$:
\begin{align*}
    n^{1-\beta} R_n &= n^{1-\beta} (R_{n, -S} - R_{n, \text{full}}) \\
    &= n^{1-\beta} \frac{s}{n} \left( \frac{\tau(1-\tau)}{2} \mathbb{E}\left[ \frac{\eta_{-S}}{f_{-S}} - \frac{\eta}{f} \right] \right) + o_p(1).
\end{align*}
Since $s \asymp n^\beta$, we have $n^{1-\beta} \cdot \frac{s}{n} \asymp n^{1-\beta} \cdot n^{\beta-1} = 1$. Therefore, the expression converges in probability to the constant $C_{\tau}(S)$.
\end{proof}

\section{Proof of Proposition \ref{prop:AsymptoticNormalityBiasCorrection} }
The decomposition is $\sqrt{n}(\tilde{V}_{\tau}(S) - V_{\tau}(S)) = \sqrt{n} A_n + \Delta_{bias} + \sqrt{n} \rho_n$.
We have already shown that $\sqrt{n} A_n \xrightarrow{d} \mathcal{N}(0, \sigma_S^2)$ and $\Delta_{bias} = o_p(1)$ under the stated condition on $\hat{C}_{\tau}$.
The critical term is the higher-order residual $\sqrt{n} \rho_n$. From Lemma 9 in Appendix C (Uniform Knight Expansion), the remainder of the quadratic expansion of the loss function satisfies:
\begin{equation*}
    |\rho_n| \le \frac{1}{n} \sum_{i=1}^n C |\hat{q}_{\tau}(X_i) - q_{\tau}(X_i)|^{2+\eta}.
\end{equation*}
Using the uniform convergence rate of the QRF estimator, $|\hat{q}_{\tau}(x) - q_{\tau}(x)| = O_p(n^{\frac{\beta-1}{2}})$, we evaluate the order of the scaled residual:
\begin{equation*}
    \sqrt{n} \rho_n = O_p\left( n^{1/2} \cdot \left( n^{\frac{\beta-1}{2}} \right)^{2+\eta} \right) = O_p\left( n^{\frac{1 + (\beta-1)(2+\eta)}{2}} \right).
\end{equation*}
For this term to converge to zero in probability, the exponent must be negative:
\begin{equation*}
    1 + (\beta-1)(2+\eta) < 0 \iff \beta < \frac{1+\eta}{2+\eta}.
\end{equation*}
Under this condition on $\beta$, the higher-order bias vanishes asymptotically, completing the proof.

\end{document}